%% file: main.tex
\documentclass[11pt]{article}

\input{header}
\usepackage[margin=1in]{geometry}
\begin{document}
\title{Learning from Censored and Dependent Data: \\
The case of Linear Dynamics}
\author{Orestis Plevrakis\\Princeton University\\\small{orestisp@princeton.edu}}
  
 \maketitle
\begin{abstract}%
  \par Observations from dynamical systems often exhibit irregularities, such as \emph{censoring}, where values are recorded only if they fall within a certain range. Censoring is ubiquitous in practice, due to saturating sensors, limit-of-detection effects, image frame effects, and combined with temporal dependencies within the data, makes the task of system identification particularly challenging. 
  \par In light of recent developments on learning linear dynamical systems (LDSs), and on censored statistics with \emph{independent} data, we revisit the decades-old problem of learning an LDS, from censored observations (\cite{lee1985common, zeger1986regression}). Here, the learner observes the state $x_t\in \R^d$ if and only if $x_t$ belongs to some set $\mS_t\subseteq \R^d$. We develop the first computationally and statistically efficient algorithm for learning the system, assuming only oracle access to the sets $\mS_t$. Our algorithm, \emph{Stochastic Online Newton with Switching Gradients}, is a novel second-order method that builds on the Online Newton Step (ONS) of \cite{hazan2007logarithmic}. Our Switching-Gradient scheme does not always use (stochastic) gradients of the function we want to optimize, which we call \emph{censor-aware} function. Instead, in each iteration, it performs a simple test to decide whether to use the censor-aware, or another \emph{censor-oblivious} function, for getting a stochastic gradient.
  \par  In our analysis, we consider a "generic" Online Newton method, which uses arbitrary vectors instead of gradients, and we prove an error-bound for it. This can be used to appropriately design these vectors, leading to our Switching-Gradient scheme. This framework significantly deviates from the recent long line of works on censored statistics (e.g, \cite{daskalakis2018efficient,kontonis2019efficient,daskalakis2019computationally}), which apply Stochastic Gradient Descent (SGD), and their analysis reduces to establishing conditions for off-the-shelf SGD-bounds. Our approach enables to relax these conditions, and gives rise to phenomena that might appear counterintuitive, given the previous works. Specifically, our method makes progress even when the current "survival probability" is exponentially small. We believe that our analysis framework will have applications in more settings where the data are subject to censoring.
\end{abstract}

\input{intro}
\input{related-work}
\input{notation}
\input{LDS}
\input{time-series}

\input{projection-set}

\input{open-questions}
\bibliographystyle{plainnat}
\bibliography{ref.bib}
\appendix
\input{appendix}

\end{document}

%% file: header.tex
\usepackage{amsmath,amsfonts,amssymb}
\usepackage{mathtools}
\usepackage{amsthm} 
\usepackage{latexsym}
\usepackage{relsize}
\usepackage{titlesec}
\usepackage[autostyle, english = american]{csquotes}
\MakeOuterQuote{"}
\newcommand{\cha}{\mathds{1}}
\usepackage{multirow,bm}
\usepackage{makecell}
\usepackage{longtable}
\usepackage{threeparttable}
\usepackage{longtable}
\usepackage{booktabs}
\usepackage[referable]{threeparttablex}
\usepackage{footnote}
\usepackage{dirtytalk}
\usepackage{caption}
\usepackage{subcaption}

\usepackage[ruled,vlined]{algorithm2e}

\usepackage{enumitem}
\usepackage{textgreek}
\usepackage{units}
\usepackage{amsmath}
\usepackage{amsthm}
\usepackage{amssymb}
\usepackage{color}
\usepackage[english]{babel}
\usepackage{graphicx}
\usepackage{grffile}
\usepackage[round]{natbib}
\usepackage{wrapfig,epsfig}
\usepackage{epstopdf}
\usepackage{url}
\usepackage{color}
\usepackage{epstopdf}
\usepackage[T1]{fontenc}
\usepackage{bbm}
\usepackage{comment}
\usepackage{dsfont}

\usepackage{tikz}
\usepackage{hyperref}  
\hypersetup{colorlinks=true,citecolor=blue,linkcolor=blue} 
\usetikzlibrary{arrows}

\usepackage[capitalise]{cleveref}
\usepackage{xcolor}
\usepackage{dsfont}

\def\H{{\mathcal H}}

\def\R{{\mathbb R}}

\newcommand{\mH}{\mathcal{H}}

\newcommand{\ignore}[1]{}

\newcommand{\email}[1]{\texttt{#1}}

\theoremstyle{plain}
	\newtheorem{theorem}{Theorem}

	\newtheorem{lemma}{Lemma}[section]
	\newtheorem{claim}[lemma]{Claim}
	
	\newtheorem{corollary}{Corollary}[section]
	\newtheorem{proposition}{Proposition}[section]

	\theoremstyle{definition}

	\newtheorem{definition}{Definition}[section]
	\newtheorem{example}{Example}[section]

	\newtheorem{assumption}{Assumption}
	
	\newtheorem{remark}{Remark}[section]

 \theoremstyle{plain}
\newtheorem*{theoremaux}{\theoremauxref}
\gdef\theoremauxref{1}

\newenvironment{proof-sketch}{%
  \renewcommand{\proofname}{Proof-sketch}\proof}{\endproof}
\DeclareMathAlphabet{\mathbfsf}{\encodingdefault}{\sfdefault}{bx}{n}

\DeclareMathOperator*{\argmin}{arg\,min}

\let\Pr\relax
\DeclareMathOperator{\Pr}{\mathbb{P}}

\newcommand{\ceil}[1]{\lceil #1 \rceil}
\newcommand{\floor}[1]{\lfloor #1 \rfloor}

\newcommand{\wt}[1]{\smash{\widetilde{#1}}}

\renewcommand{\O}{O}

\newcommand{\tO}{\wt{\O}}

\newcommand{\trace}{\mathrm{tr}}
\newcommand\simiid{\mathrel{\overset{\makebox[0pt]{\mbox{\normalfont\tiny\sffamily i.i.d}}}{\sim}}}

\newcommand{\poly}{\mathrm{poly}}

\newcommand{\Del}{\Delta}

\renewcommand{\leq}{~\le~}
\renewcommand{\geq}{~\ge~}

\let\oldtfrac\tfrac
\renewcommand{\tfrac}[2]{\smash{\oldtfrac{#1}{#2}}}

\let\nablaold\nabla
\renewcommand{\nabla}{\nablaold\mkern-2.5mu}

\newcommand{\mT}{\mathcal{T}}
\newcommand{\mR}{\mathcal{R}}
\newcommand{\mP}{\mathcal{P}}
\newcommand{\mO}{\mathcal{O}}
\newcommand{\mF}{\mathcal{F}}

\newcommand{\mN}{\mathcal{N}}
\newcommand{\mS}{\mathcal{S}}
\newcommand{\mI}{\mathcal{I}}
\newcommand{\mK}{\mathcal{K}}

\newcommand{\mL}{\mathcal{L}}
\newcommand{\mG}{\mathcal{G}}
\newcommand{\mE}{\mathcal{E}}

\newcommand{\mB}{\mathcal{B}}
\newcommand{\Ex}{\mathbb{E}}

\newcommand{\ga}{\Gamma}
\newcommand{\mus}{\mu^*}
\newcommand{\sA}{A_*}

\newcommand{\whA}{\widehat{A}}

\newcommand{\Si}{\Sigma}
\newcommand{\sumt}{\sum_{t=1}^T}
\newcommand{\dmus}{\|\mu_t-\mu_t^*\|}

\newcommand{\prt}{\mN(\mu_t,I,S_t)}

%% file: intro.tex
\section{Introduction}\label{sec:intro}
\par System identification is the problem of learning the evolution equations of a dynamical system from data. Mathematically, we have a sequence of system states $(x_t)_t$, and observations $(y_t)_t$ evolving as 
\begin{align*}
    x_{t+1}=f(x_t,u_t,w_t), \ \ y_t=g(x_t,v_t),
\end{align*}
where $u_t$'s are inputs to the system, $w_t$ is process noise, and $v_t$ is sensor noise. At each step $t$, the learner observes the input $u_t$, and the resulting output 
$y_{t+1}$. The goal is to learn the functions $f$ and $g$, from an observed trajectory.\footnote{Some works consider access to multiple trajectories.} In this paper, we consider system identification with \emph{censored} observations, where the learner observes $y_t$ if and only if $y_t$ belongs to some set $\mS_t$, which we call the \emph{observable} set.
\begin{figure}
     \centering
     \begin{subfigure}[b]{0.4\textwidth}
         \centering
         \includegraphics[width=\textwidth]{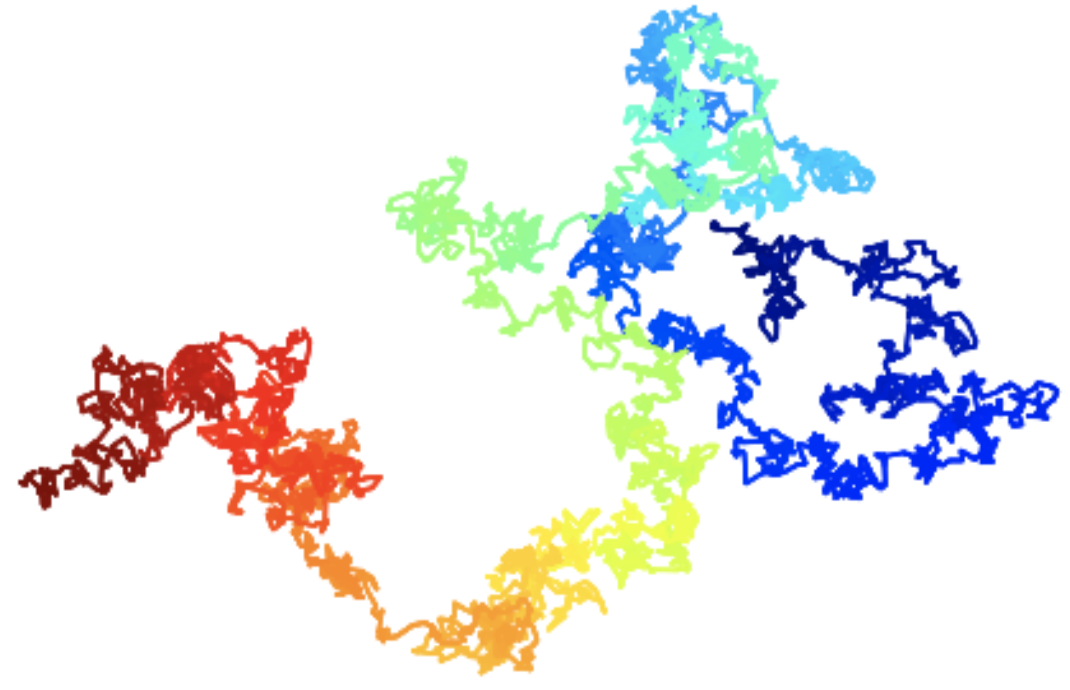}
         \caption{$\mS_t=\R^2$}
         \label{fig:eg-full}
     \end{subfigure}
     \hspace{1cm}
     \begin{subfigure}[b]{0.4\textwidth}
         \centering
         \includegraphics[width=\textwidth]{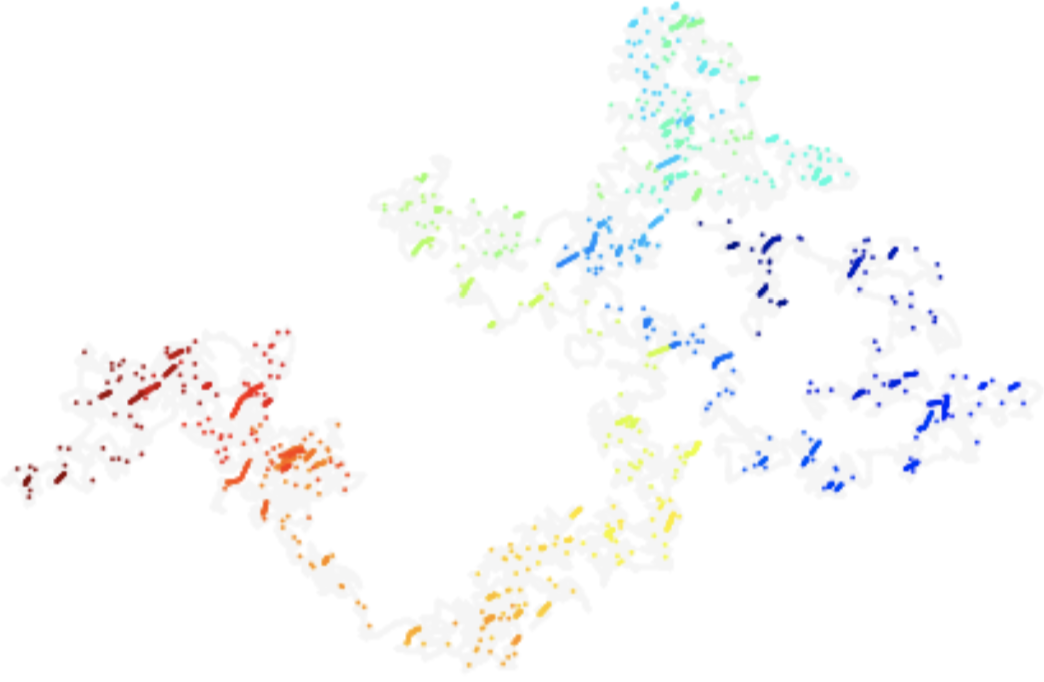}
         \caption{$\mS_t=\{(x_1,x_2): x_1\ge x_{t,2}+0.4,x_2\ge x_{t,1}+0.2\}$ }
         \label{fig:eg-cens}
     \end{subfigure}
        \caption{Linear system in $2$-$d$, $A_*=\rho \cdot I$, $\rho=1-\epsilon$, $\epsilon\approx 0$. Figure \ref{fig:eg-full} is an uncensored trajectory (color is a function of time). Figure \ref{fig:eg-cens} is the same trajectory, but the grey areas are censored.}
        \label{fig:lds-eg}
\end{figure}
\par Dynamical data with missing observations are ubiquitous in practice \citep{honaker2010missing}. This is often due to censoring, which frequently manifests in the fields of signal processing and control theory \citep{yang2009set}, time series analysis \citep{lee1985common}, business \citep{hausman1977social}, economics \citep{johannsen2018time}, in medicine, and in physical sciences. For example, consider learning the dynamics of some target dynamical system, using observations from cameras. Here, censoring naturally arises due to occlusions blocking visibility, or from the system exiting the camera frame. Despite the numerous applications, the problem is not statistically and computationally understood even for linear, fully-observable dynamics: 
\begin{align*}
    x_{t+1}=A_*x_t+B_*u_t+w_t,\ \ y_t=x_t,
\end{align*}
where $x_t \in \mathbb{R}^d$, $u_t\in \mathbb{R}^m$, and $w_t\simiid N(0,I)$.  Actually, even for one-dimensional linear systems $(d=1)$, with no inputs $(u_t=0)$, and observable sets $\mS_t$ being half-lines, no known efficient algorithm learns (the scalar) $\sA$, not even asymptotically.  Specifically, the existing methods fall in two categories: 1) iterative methods trying to maximize the non-concave log-likelihood, as in \cite{lee1985common} and \cite{zeger1986regression}, and 2) EM-based imputation methods (e.g., \cite{park2007censored}). Both of them are not guaranteed to recover the underlying system. On top of that, for large dimension $d$, the likelihood-based approach is inefficient to implement because it requires computing high-dimensional integrals over $\mS_t$'s. 
\par In this work, we study the multidimensional linear case with no inputs: $ x_{t+1}=A_*x_t+w_t$ (e.g., Figure \ref{fig:eg-full}). We allow \emph{arbitrary} observable sets $\mS_t$ (e.g., Figure \ref{fig:eg-cens}), with only requirement being that we have oracle access to each $\mS_t$, namely, given a point $x$, the oracle efficiently computes $\cha\{x\in \mS_t\}$. For this model, we obtain the first computationally and statistically efficient algorithm for learning $\sA$, under the following assumptions (stated informally here, and in detail in Section \ref{sec:LDS}):
\begin{description}
\item[Assumption 1:] The system is \emph{stable}, i.e., the spectral radius $\rho(\sA)$ is less than one. Stability is a classical assumption in linear dynamical systems (LDSs). If $\rho(\sA)>1$, then the state explodes exponentially, with high probability. 
\item[Assumption 2:] The number of times we observe the state is at least a constant fraction (say $1\%$) of the trajectory-length.
\item[Assumption 3:] For most of the timesteps $t$, given that we observed $x_t$, the probability of observing $x_{t+1}$ is at least a constant (say $1\%$). 
\end{description}
To motivate the last two assumptions, we note that for the simpler problems of censored Gaussian estimation and censored linear regression (with independent data), the only known computationally and statistically efficient algorithms assume that the probability of observing a sample is at least a constant (see \cite{daskalakis2018efficient}, \cite{daskalakis2019computationally}). Assumptions 2 and 3 are the adaption of this, to fit our dynamical setting. Under these assumptions, our estimation error bound matches (up to logarithmic factors) the best known bound for learning uncensored LDSs (\cite{simchowitz2018learning}), with respect to the dimension $d$, the trajectory-length $T$, and the spectrum of $\sA$. It is also the first, to the best of our knowledge, estimator that can accommodate arbitrary observable sets, since previous works considered intervals, half-lines, and products of these (e.g., \cite{zeger1986regression}, \cite{yang2009set}). Key for tackling this decades-old problem are recent advances in 1) (uncensored) linear system identification theory, and 2) censored/truncated statistics (CTS) for independent data.
\paragraph{Learning Linear Dynamics.}The difficulty in learning an LDS, compared to standard linear regression, is that the observations are dependent. A classical approach that avoids this issue is based on the system's mixing time, which is roughly $\tau_{\text{mix}}=\frac{1}{1-\rho(\sA)}$ steps. Thus, the learner can use one $x_t$ every $\tau_{\text{mix}}$ steps, and reduce the analysis to standard linear regression (e.g., \cite{yu1994rates}). However, the resulting bounds get worse with larger $\tau_{\text{mix}}$, and as pointed in \cite{simchowitz2018learning}, this behavior is qualitatively incorrect. Intuitively, larger $\sA$ gives larger states, which implies larger signal-to-noise ratio, i.e., easier estimation. The authors provided bounds that express this intuition, and that was the first sharp analysis for stable systems.\footnote{The authors also considered $\rho(\sA)=1$, known as the \emph{marginally stable} regime.} This progress initiated a large line of work on learning LDSs (see Section \ref{sec:related}).  

\paragraph{Censored/Truncated Statistics (CTS).} Consider the problem of learning a Gaussian distribution, having access only to samples from some set $\mS$. These are called \emph{truncated} samples. Censoring is when we also know the number of unobserved samples, as in our case, where this number can be inferred from the lengths of time-intervals during which we do not observe anything. Truncation and censoring go back to at least \cite{galton1898examination} and \cite{pearson1902systematic}, and there has been a large volume of research devoted to them (see \cite{cohen1991truncated}). Nevertheless, the first provably computationally and statistically efficient algorithm, for learning a truncated Gaussian, was only recently discovered by \cite{daskalakis2018efficient}. The authors developed a general algorithmic framework, based on Stochastic Gradient Descent (SGD), which bypasses the computation of high-dimensional integrals over $\mS$. The result and the generality of the approach created a lot of excitement, and a large number of subsequent works applied the SGD framework to other problems in CTS (for independent data), e.g., linear regression by \cite{daskalakis2019computationally}. 

\subsection*{Our Contributions}
We build on the above advances, by introducing new algorithmic and technical ideas, which we now overview.
\paragraph{Our algorithm.} 
Our first observation is that non-convexity of the negative log-likelihood can be bypassed by focusing on "paired observations", i.e., $(x_t,x_{t+1})$ such that $x_t\in \mS_t$ and $x_t\in \mS_{t+1}$. By ignoring the other terms in the objective, we get a convex function. The second observation is that for LDSs, if an SGD-based algorithm relies on the Markovian property $\mathbb{E}[x_{t+1}| x_t,x_{t-1},\dots,x_1]=\sA x_t$ to produce \emph{unbiased} gradient estimates, then the algorithm should process the data in \emph{temporal order}. The reason is that if we "see" $x_t$ and $x_{t+s}$ (for some $s\ge 2$), then the expectation of $x_{t+1}$ is not $\sA x_t$, i.e., $\mathbb{E}[x_{t+1}| x_t,x_{t+s}]\neq \sA x_t$. In other words, we need an $\emph{online}$ algorithm. Unfortunately, for censored linear regression, the SGD framework of \cite{daskalakis2019computationally} processes the data in \emph{random} order, and also does multiple passes over them. This is not a technicality; making that algorithm online will lead to slow statistical rates, as we explain later. For this reason, we design a new stochastic \emph{second-order} method, building on the Online Newton Step method (ONS) of \cite{hazan2007logarithmic}. The crucial difference with ONS, and all recent works on CTS, is that we do not always use (stochastic) gradients of the function we want to optimize, which we call \emph{censor-aware} function. Instead, in each iteration, we perform a simple test, based on which we decide whether to use the censor-aware function, or another \emph{censor-oblivious} function, for getting a stochastic gradient. We call our method \emph{Stochastic Online Newton with Switching Gradients} (SON-SG, given in Algorithm \ref{alg:newton}), and we show that it can be combined with a least-squares-based warmup procedure, to learn censored LDSs (Algorithm \ref{alg:overall}). We also show that SON-SG can be applied to an even broader setting, i.e., general linear-response time-series (Section \ref{sec:time-series}).
\paragraph{Algorithm-design framework.} We came up with this switching-gradient scheme, by considering a "generic" ONS method, where instead of gradients we have an arbitrary vector sequence $g_1,g_2,\dots$. We proved a general estimation error bound for this method, which serves as a guideline for designing the $g_i$'s. A similar framework for designing $g_i$'s has been proposed in the context of non-convex optimization (for first-order methods), by \cite{arora2015simple}. This approach gives a lot of freedom for algorithm design, compared to previous works on CTS which apply off-the-shelf SGD-bounds, and we believe it will have more applications in censored and truncated statistics.
\paragraph{Technical contributions and insights.} In all recent CTS papers (see Section \ref{sec:related} for an extensive list), a crucial step for proving parameter recovery by SGD is establishing \emph{anti-concentration} of truncated Gaussians $\mN(\mu,\Sigma,\mS)$, where $\mN(\mu,\Sigma,\mS)$ is the Gaussian $\mN(\mu,\Sigma)$, given that the sample is in $\mS$.  \cite{daskalakis2018efficient} reduced this task to showing that the "survival probability" is large, i.e., $\Pr_{x\sim \mN(\mu,\Sigma)}[x\in \mS]\ge \Omega(1)$. Hence, this and all follow-up works have focused on lower bounding survival probabilities. In our case, this methodology does not apply, due to temporal dependencies. However, our approach (generic ONS bound) enables to significantly relax the high survival probability condition, and gives rise to phenomena that might appear counterintuitive, given the intuition built in the recent literature. Specifically, in high dimensions ($d \to \infty$) our analysis deals with cases where the survival probability is exponentially small $(e^{-\Omega(d)})$, while at the same time anti-concentration tends to infinity.  Our other technical contribution is a lower-bound on the covariance matrix of the \emph{observed} states. \cite{simchowitz2018learning} proved such a bound when all states are observed. Here, due to censoring, we observe only a subset of the whole trajectory. For independent data, \cite{daskalakis2018efficient} address this issue using union-bound over all possible subsets.\footnote{This is implicit in their analysis. They use a result of \cite{diakonikolas2019robust}, which applies union bound over subsets.} Unfortunately, for LDSs union-bound gives vacuous guarantees, because for a fixed subset, the concentration degrades with larger mixing times. We resolve this difficulty, by generalizing the "small-ball" technique of \cite{simchowitz2018learning}, and we lower-bound the covariance matrices of \emph{all subsets} of size $\Omega(T)$ \emph{simultaneously}, where $T$ is the trajectory-length. 

%% file: related-work.tex
\subsection{Further Related Work}\label{sec:related}
\paragraph{Censored and Truncated Statistics.}
As we mentioned, there has been a long line of recent works on several settings within CTS: Gaussian parameter estimation \citep{daskalakis2018efficient,kontonis2019efficient}, linear, logistic and probit regression \citep{daskalakis2019computationally,ilyas2020theoretical}, compressed sensing \citep{daskalakis2020truncated}, sparse graphical models \citep{bhattacharyya2020efficient}, estimation of boolean product distributions \citep{fotakis2020efficient}, mixtures of Gaussians \citep{nagarajan2020analysis}. All these works consider independent data and apply the SGD framework\footnote{Exception is \cite{nagarajan2020analysis} who consider the EM-algorithm.} introduced in \cite{daskalakis2018efficient}.
\paragraph{Linear System Identification.} 
Even though linear system identification is a decades-old field \citep{ljung1999system}, a sharp non-asymptotic theory was only recently developed \citep{simchowitz2018learning,sarkar2019near,tsiamis2019finite,oymak2019non,simchowitz2019learning}.
\paragraph{Online Convex Optimization.}
To design our algorithm, we build on ideas from online convex optimization (OCO). In the recent years, OCO has been extensively used for learning and controlling LDSs (e.g., \cite{agarwal2019online,hazan2020nonstochastic, simchowitz2020improper, ghai2020no, simchowitz2020making}). For a general overview of OCO see \cite{hazan2019introduction}.

%% file: notation.tex
\section{ Notation}
\par For every vector $x$, we use $\|x\|$ to denote the $\ell_2$ norm $\|x\|_2$. Also, we use $\langle A,B\rangle=\trace(A^\top B)$ to denote the matrix inner product. For a matrix $A$ and a $\Sigma \succ 0$, we define the norm $\|A\|_{\Sigma}= \sqrt{\langle A^\top 
A,\Sigma \rangle}$ \footnote{Usually, $\|A\|_{\Sigma}$ is used to denote $\sqrt{\langle AA^T,\Sigma \rangle}$. However, this definition is more appropriate for our setting.}. The covariance matrix between two random vectors
$x, y$ is $\text{Cov}[x,y]$. For a sequence $(x_t)_{t=1}^T$, we denote by $x_{\le \tau}$,  $x_{<\tau}$ and $x_{-\tau}$ the subsequences $(x_t)_{t\le \tau},(x_t)_{t< \tau}, (x_t)_{t\neq \tau}$ respectively. For a Gaussian distribution $\mN(\mu,\Sigma)$ in $\R^d$, and a measurable $\mS\subseteq \R^d$, we define the survival probability $\mN(\mu,\Sigma;\mS)=\Pr_{x\sim \mN(\mu,\Sigma)}[x\in \mS]$. We also define the truncated Gaussian $\mN(\mu,\Sigma,\mS)$ to be $\mN(\mu,\Sigma)$ conditioned on taking values in $\mS$. Finally, whenever we say that a set $\mS$ is "revealed" to the learner, we mean that she has access to a membership oracle $M_{\mS}$, i.e., an efficient procedure that computes the $\cha\{x\in \mS\}$, for any point $x$.  

%% file: LDS.tex
\section{Censored Linear Dynamics: Model, and Main Theorem}\label{sec:LDS}
We study the system $x_{t+1}=\sA x_t+w_t$, where $x_t \in \R^d$, $\sA\in \R^{d\times d}$ and  $w_t\simiid \mN(0,I)$. Starting from $x_0=0$,\footnote{We assume $x_0=0$ to simplify the exposition.  Our proofs generalize for any $x_0$, by paying a $\log{\|x_0\|}$ factor in the bound.} consider the trajectory $x_1,x_2,\dots,x_{T+1}$. The learner has access to censored observations, i.e., there is a process of observable sets $(\mS_t)_t$, and she observes $x_t$ if and only if $x_t\in \mS_t$. Also, at time $t$, the set $\mS_t$ is revealed to her. Now, $\mS_t$'s may depend on the state-trajectory, but we assume that given $x_t$, the state $x_{t+1}$ and the set $\mS_{t+1}$ are statistically independent. To see why this is a natural assumption, consider the camera-based example (Section \ref{sec:intro}), and think of $x_t \in \R^{3}$ as the position of some object, and $\mS_t$ as the visible part of the space at time $t$. Having observed $x_t$, the camera could adapt its frame (affecting $\mS_{t+1}$), to improve the chances for observing $x_{t+1}$, but without knowing the next ``excitation'' $w_t$. We now formally state our assumptions, sketched in Section \ref{sec:intro}.
\begin{assumption}\label{ass:stability}
  $\sA$ is diagonalizable and stable, i.e., $\sA=UDU^{-1}$, where $D$ is diagonal and $\rho(\sA)=\max_{i} |D_{ii}|< 1$.\footnote{Note that $S$ and $D$ can have complex entries.}
\end{assumption}

Let $\mO$ be the set of observation times: $\mO=\{t \in [T]: x_{t}\in \mS_t\}$. We assume that we observe "enough" data: 
\begin{assumption}\label{ass:observable}
 For a known constant $\beta\in (0,1)$, with probability $1-o(1)$, we have $|\mO|\ge \beta T$, where $o(1)$ denotes a $\delta_T\to 0$, as $T\to \infty$.
\end{assumption}
Let $\mB(a)$ be the set of timesteps $t$, at which 1) we observe $x_t$, and 2) given $x_t$ and $\mS_{t+1}$, the probability of observing $x_{t+1}$ is less than $a$, i.e., $\mB(a)=\left\{t\in \mO :\ \mN\left(\sA x_t,I;\mS_{t+1}\right)<a \right\}$. 
\begin{assumption}\label{ass:survival-prob}
For a known constant $\alpha \in (0,1)$, and some bound $L>0$, with probability $1-o(1)$, we have $|\mB(\alpha)|\le L$. Also, $\Ex \big[|\mB(\alpha)|\big]\le L$.
\end{assumption}
Our bounds will depend on $L$, and will match the uncensored case if $L\le \tO(d)$.\footnote{$\tO(\cdot)$ hides logarithmic factors.} As we mentioned in the introduction, Assumptions \ref{ass:observable} and \ref{ass:survival-prob} are the adaptation of the $\Omega(1)$-survival-probability assumption in \cite{daskalakis2018efficient}. We further motivate Assumption \ref{ass:survival-prob} with a natural one-dimensional example.
\begin{example}
Let $x_{t+1}=a_*x_t+w_t$, where $a_*\in[0,1)$. The observable set is a static half-line: $\mS_t=\mS=\{x\in \R\ :\ x\ge \lambda\}$, $\lambda>0$. We claim that here, Assumption \ref{ass:observable} implies Assumption \ref{ass:survival-prob} with $\alpha=\Omega(1)$ and $L=0$. This is clear for $a_* \approx 1$, since if $x_t \ge \lambda$, then $a_* x_t +w_t \ge x_t\ge \lambda$ with probability almost $1/2$.\footnote{Note that for $a_* \approx 1$, we did not need Assumption \ref{ass:observable}, to show Assumption \ref{ass:survival-prob}. However, we need it for general $a_*$.} For general $a_*\in[0,1)$, the implication is less obvious (see Appendix \ref{appdx:1d}). 
\end{example}
 As in \cite{simchowitz2018learning}, our bounds depend on the controllability Gramian $\ga_T$, defined as $\Gamma_T\coloneqq \sum_{s=0}^{T-1}\sA^s(\sA^s)^\top$. This matrix quantifies how much the noise process excites the system. In the theorem that follows, we use $\widetilde{\Theta}(1)$ to denote polylogarithmic factors in $T$ and in $\text{cond}(U)$, where $\text{cond(U)}$ is the condition number of the eigenvector-matrix $U$.
\begin{theorem}\label{thm:main}
Under Assumptions \ref{ass:stability}, \ref{ass:observable}, \ref{ass:survival-prob}, there exist $C_{\alpha,\beta}, C_{\alpha,\beta}'=\widetilde{\Theta}(1)\cdot \poly\left(\frac{1}{\alpha \beta}\right)$, such that if 
\begin{align*}
 T\ge C_{\alpha,\beta}'\cdot \left(d^2+\frac{d}{1-\rho(\sA)}+dL\right),   
\end{align*}
then with probability at least $99\%$, Algorithm \ref{alg:overall} runs in polynomial time, and outputs an $\whA$ such that 
\begin{align}\label{eq:error-bound-main}
   \left\|\whA-\sA\right\|_{\Gamma_T}\leq C_{\alpha,\beta}\sqrt{\frac{d^2+dL}{T}}.
\end{align}
\end{theorem}
\begin{remark}
For uncensored LDSs, and stable-diagonalizable $\sA$, the best known (spectral-norm) bound is $\left\|(\whA-\sA)\ga_T^{1/2}\right\|_2\le \tO\left( \sqrt{d/T}\right)$. So, the Frobenius version is $\|\whA-\sA\|_{\ga_T}\le \tO\left( \sqrt{d^2/T}\right)$, which matches (\ref{eq:error-bound-main}), if $\alpha,\beta=\Omega(1)$, and $L=\tO(d)$. We leave the spectral-norm bound for censored LDSs for future work.
\end{remark}
\begin{remark}
It is possible to drop the assumption that $\sA$ is diagonalizable by paying an exponential dependence in the size of the largest Jordan block. This dependence (for a bound on $\|\cdot\|_{\Gamma_T}$) appears even in the uncensored case (see \cite{ghai2020no} for a discussion on this).
\end{remark}
As we mentioned, our algorithm is SON-SG, preceded by a least-squares warmup procedure. Let $\mP$ be the set of of pairs $(x_t,x_{t+1})$, such that both $t,t+1\in \mO$. We will ignore all "isolated" observations, i.e., the ones not participating in any pair of $\mP$. As we will see, by ignoring them This decision is justified by the following proposition:
\begin{proposition}\label{prop:largeN}
Let $M\coloneqq|\mP|$. With probability $1-o(1)$, we have $M \ge \alpha\beta T /2$.
\end{proposition}
 We prove Proposition \ref{prop:largeN} in Appendix \ref{appdx:pairs}, using Assumptions \ref{ass:observable},\ref{ass:survival-prob}, and that $T>>L$. Let $\mP_0$ be the first (in time-order) $\floor{M/2}$ pairs of $\mP$, and $\mP_1=\mP\setminus\mP_0$. Our algorithm first uses $\mP_0$ to create a (rough) confidence ellipsoid $
\mK$ that includes $\sA$, with high probability. Then, it uses $\mK$ as constraint-set for SON-SG, which will operate on $\mP_1$. In the next section, we present and analyze SON-SG in a more general setting, which reveals the key structure that our method exploits. Then, in Section \ref{sec:projection}, we give the full-algorithm and conclude the proof of Theorem \ref{thm:main}.

%% file: time-series.tex
\section{Truncated Time Series with Linear Responses}\label{sec:time-series}
 For timesteps $1$ to $T$, consider a covariate-response process $(x_t,y_t)_{t}$, where $x_t\in \R^d$, $y_t\in \R^n$. The learner only observes a subset of the data, based on a Bernoulli process $(o_t)_t$, i.e., if $o_t=1$, then she observes $(x_t,y_t)$, otherwise the pair is hidden. In untruncated linear-response time series, we have $y_t=\sA x_t +w_t$, where $w_t\sim N(0,I)$, and $\sA\in \R^{n\times d}$. Here, when the learner gets to see a datapoint, the noise will be biased due to truncation. Formally, consider a process of observable sets $(S_t)_t$, where $S_t\subseteq \R^n$. Let $\mF_t$ be the $\sigma$-algebra generated by $x_{\le t}$, $y_{< t}$, $o_{\le t}$ and $S_{\le t}$ (note that $y_t$ is not in $\mF_t$). We assume that given $\mF_t$ and $o_t=1$, the set $S_t$ is revealed to the learner, and $y_t\sim \mN(\sA x_t,I,S_t)$. This model has three notable special cases:

\begin{enumerate}
    \item \textbf{Untruncated, linear-response time series.} This model was studied by \cite{simchowitz2018learning}, and corresponds to $o_t=1$ and $S_t=\R^n$, for all $t$.
    \item \textbf{Truncated linear regression.} First considered by \cite{tobin1958estimation}, this model was revisited by  \cite{daskalakis2019computationally}, and corresponds to $o_t=1$ for all $t$, and it requires independent data, i.e., given $x_t$, the response $y_t$ is independent of $x_{-t}$.\footnote{To be precise, \cite{daskalakis2019computationally} consider one-dimensional responses and fixed truncation set. However, the extension to multidimensional $y_t$'s and time-varying truncation sets is relatively straightfoward.}
    \item \textbf{LDS with censored observations.} Here, $y_t=x_{t+1}$, $o_t=\cha \left\{x_t\in \mS_t\ \land\ x_{t+1}\in \mS_{t+1}\right\}$, and $S_t=\mS_{t+1}$. Note that we pretend we do not observe the "isolated" observations, which is aligned with what our algorithm does. 
\end{enumerate}

\paragraph{Initial confidence ellipsoid.} SON-SG receives as input a rough initial estimate $A_0\in \R^{n \times d}$, and a $\Sigma_0 \succ 0$ that represents a confidence ellipsoid $\mK=\{A\in \R^{n\times d}:\ \|A-A_0\|_{\Sigma_0}\le 1\}$. For LDSs, we will later show how to use $\mP_0$ to construct a $\mK$ with the following property:
\begin{definition}
A confidence ellipsoid $\mK=\{A\in \R^{n\times d}:\ \|A-A_0\|_{\Sigma_0}\le 1\}$ is $(R,\omega)$-accurate, for some $R,\omega >0$, if (a) $A_*\in \mK$, (b) $ \Sigma_0\succcurlyeq  \omega \cdot I$, and (c) for all $t$, $\big\|\Sigma_0^{-1/2}x_t\big\|$ is $R^2$-subgaussian.\footnote{A random variable $X$ is called $\sigma^2$-subgaussian if $\Pr[|X|\ge \delta\cdot \sigma]\le \exp\left(-\delta^2/C\right)$, where $C=O(1)$.}
\end{definition}
\begin{theorem}\label{thm:time-series}
Fix $R,R_w,R_x,L,\alpha>0$, and let $B(\alpha)\coloneqq\{t\in[T]:\ o_t=1,\ \mN(\sA x_t,I;S_t)<\alpha\}$. Suppose that (a) we are given an $(R,\omega)$-accurate confidence ellipsoid $\mK$, (b) $\Ex[|B(\alpha)|]\le L$, (c) for all $t$, the noise $w_t=y_t-\sA x_t$ has norm $\|w_t\|$ that is $R_w^2$-subgaussian, and (d) $\Ex[\|x_t\|^2]\le R_x^2$. Let $t_1<t_2<\dots<t_N$ be the observation times ($o_{t_i}=1$). If the total number of steps $T\ge \poly(1/\alpha)$, then SON-SG (Algorithm \ref{alg:newton}) outputs an $\whA$ such that 
\begin{align}\label{eq:error-bound}
   \mathbb{E}\left[\left\|\whA-\sA\right\|_\Sigma ^2\right]\leq \frac{D+L D'}{N},
\end{align}
where $\Sigma=\frac{1}{N}\sum_{i=1}^Nx_{t_i}x_{t_i}^T$, and $D=\tO(1)\cdot dD'$,  $D'=\tO(1)\cdot \poly(1/\alpha)\cdot(d+R^2+R_w^2)$. Here, $\tO(1)$ denotes a polylogarithmic factor in the parameters defined in this section.
\end{theorem}
\begin{remark}
Note that in the LDS case, $R_w=O(\sqrt{d})$ by standard concentration of Gaussian norm. However, for slowly mixing systems ($\rho(\sA)\to 1$), $R_x$ can grow polynomially with $T$. Our bounds only degrade in the logarithm of $R_x$, which is absorbed in $\tO(\cdot)$, and the same happens with $1/\omega$. 
\end{remark}
Before presenting SON-SG and its analysis, we give some background on existing techniques for CTS.

\subsection{Existing Techniques and their Limitations}
Even though \cite{daskalakis2019computationally} consider truncated linear regression with no temporal dependencies, we will use the same likelihood-based objective. Specifically, let 
\begin{align*}
    \ell_S(\mu;y)\coloneqq -\frac{1}{2}\|y-\mu\|^2 -\log\left(\int_S \exp\left(-\frac{1}{2}\left\|z-\mu\right\|^2\right) dz\right),
\end{align*}
and observe that given $\mF_t$ and that $o_t=1$, we have that for a candidate matrix $A$, the log-likelihood for $y_t$ is $\ell_{S_t}(Ax_t ; y_t)$.
Also, let $f_t(A)$ be the (negative) population log-likelihood:
\begin{align*}
f_t(A)\coloneqq -\mathbb{E}_y\Big{[}\ell_{S_t}(Ax_t;y)\ \Big{|}\ \mF_t,o_t=1\Big{]},
\end{align*}
and observe that $f_t(A)=-\mathbb{E}_{y\sim \mN\left(A_*x_t,I,S_t\right)}\Big{[}\ell_{S_t}(Ax_t;y)\Big{]}$. Given the observed data, our goal will be to minimize $ f(A)\coloneqq \frac{1}{N}\sum_{i=1}^N f_{t_i}(A)$.
To see why $f(A)$ is a ``good'' objective, we take the first and second derivatives of $f_t(A)$:\footnote{$\otimes$ denotes Kronecker product.}
\begin{align*}
    \nabla f_t(A)= \mathbb{E}_{z\sim \mN\left(Ax_t,I,S_t\right)}\left[z\right]x_t^\top- \mathbb{E}_{y\sim \mN\left(A_*x_t,I,S_t\right)}\left[y\right]x_t^\top
\end{align*}
\begin{align}\label{eq:cov}
    \nabla^2 f_t(A)= \text{Cov}_{z\sim \mN\left(Ax_t,I,S_t\right)}\left[z,z\right]\otimes \left(x_tx_t^\top\right),
\end{align}
where $\nabla^2 f_t(A)$ is used to denote $\nabla^2 f_t\left(\text{vec}\left(A^\top\right)\right)$, and $\text{vec}(\cdot)$ is the standard vectorization. Now, note that $\nabla f_t(\sA)=0$, and so $\nabla f(\sA)=0$. Also, since the Kronecker product of positive semidefinite matrices (PSD) is PSD, $f_t(A)$ is convex, and so $f(A)$ is also convex.\footnote{As we mentioned in the introduction, for censored LDSs, the negative log-likelihood is non-convex. However, here we have convexity, because $f(A)$ corresponds to a part of the overall log-likelihood.} Now, if $f$ was \emph{strongly-convex} ($\nabla^2 f(A)\succ 0$), then $\sA$ would be the unique optimal solution, justifying the use of the objective. Even though strong-convexity does not hold, \cite{daskalakis2019computationally} show how to address this for independent data, by restricting $A$ in some set that contains $\sA$, and $\nabla^2 f(A)\succeq \Omega(1)\cdot I$, inside the set. Let's assume (for now) that here, $\nabla^2 f(A)\succeq \Omega(1)\cdot I$ holds. Now, note that a priori, it is not clear how to optimize $f(A)$, since we do not have a closed-form expression. An important conceptual contribution of \cite{daskalakis2019computationally} is the observation that by sampling $z_{t}\sim \mN\left(Ax_t,I,S_t\right)$,\footnote{This is done via rejection sampling, using the membership oracle for $S_t$.} and computing $v_t=(y_t-z_t)x_t^\top$, we have $\mathbb{E}\big{[}v_t\ \big{|}\ \mF_t,o_t=1\big{]}=\nabla f_t(A)$, i.e, we get an unbiased gradient estimate. Based on this observation, they employ a variant of SGD to minimize $f(A)$, but their algorithm is tailored to independent data. The reason is that it processes the data in random order, and also does multiple passes over them. Thus, by the time it computes $v_t$, it is very likely that before that, it had processed $(x_t',y_t')$, for $t'>t$. Because of this, if the data are temporally dependent, $v_t$ can be a \emph{biased} estimate of $\nabla f_t(A)$. 

\subsection{Our Approach}
To avoid the above issue, we need to process the data in temporal order. This is exactly the case for Online Convex Optimization (OCO). In OCO though, the goal is not to recover some parameter, but to minimize \emph{regret}. However, regret bounds can often be transformed to statistical recovery rates via "online-to-batch" conversions (e.g., \cite{cesa2004generalization}). In our setting, this conversion can be done, but is trickier than usual, and we will deal with it later. Now, since we are aiming for a fast $\tO(1/N)$-rate, the first natural attempt is online SGD, which has only logarithmic (in $N$)  regret, provided all $f_{t_i}$'s are strongly-convex (\cite{hazan2019introduction})\footnote{In online-to-batch conversions, a regret-bound $R_N>0$ often translates to $R_N/N$ statistical recovery rate.}. Unfortunately, this in not true for any $f_{t_i}$, due to the rank-one component $x_{t_i}x_{t_i}^T$ in the Hessian $\nabla^2 f_{t_i}(A)$ (\ref{eq:cov}). However, there is still structure we can exploit. Notice that every row of $\nabla f_{t_i}(A)$ has the same direction as $x_{t_i}$, which corresponds exactly to that "problematic" rank-one component in the Hessian. This structure is reminiscent of the \emph{exp-concavity} property (\cite{hazan2007logarithmic}), which essentially is strong-convexity, \emph{in the direction of the gradient}: 
\begin{definition}
A function $f$ is called $\lambda$-exp-concave, if for all $x$, $ \nabla^2 f(x) \succcurlyeq \lambda \cdot \nabla f(x) \nabla f(x)^\top$.\footnote{These functions are called exp-concave, because this property is equivalent to $e^{-\lambda f(x)}$ being concave.}
\end{definition}
For exp-concave functions, the Online Newton Step (ONS) algorithm, introduced in \cite{hazan2007logarithmic}, has regret that depends logarithmically in $N$. Unfortunately, this result does not apply here:
\paragraph{Obstacles for ONS.} First, $f_{t_i}$'s are not necessarily exp-concave (unless $\lambda$ is exponentially small, which is not useful). This is because of the covariance term in \ref{eq:cov}. To the best of our knowledge, the idea used in \cite{daskalakis2019computationally} to restrict $A$ is some appropriate set, does not resolve this issue, due to temporal dependencies. The second obstacle, is that the regret bound of ONS (\cite{hazan2007logarithmic}) will have linear dependence in $R_x$, which as we said can grow as $\poly(T)$.

\subsection{Stochastic Online Newton with Switching Gradients}\label{sec:algo}
We now describe SON-SG. First, we use as projection-set the ellipsoid $\mK$. Second, we use preconditioning as in ONS, but while in ONS the preconditioner has outer-products of the gradients, here we use outer-products of the covariates. This is done to simplify the analysis. 
\paragraph{Switching Gradients.} The crucial difference with ONS is the choice of the $g_i$'s, by the "SwitchGrad" function. To ease notation, suppose we are at time $t$, and $t=t_i$ for some $i$. We define $A(t)= A_{i}$, and $g(t)= g_i$. In ONS, $g(t)$ would simply be $\nabla f_t(A(t))$. Of course, we do not have access to this gradient, but as we said we can get a stochastic gradient by sampling $z_t\sim \mN(A(t)x_t,I,S_t)$, and setting $g(t)=(z_t-y_t)x_t^\top$. Sampling from $\mN(A(t)x_t,I,S_t)$ can be done via rejection sampling, using the membership oracle. However, to be efficient, the mass $\gamma_t\coloneqq \mN(A(t)x_t,I;S_t)$ should be sufficiently large. Unfortunately,  here $\gamma_t$ can be exponentially small. This case though is easily recognizable, by estimating $\gamma_t$ via sampling from the normal $\mN(A(t)x_t,I)$, and counting how many times we hit $S_t$, using the membership oracle. This is done by the "Test" function. If the Test returns "True", then with high probability (w.h.p), $\gamma_t \ge \alpha^{O(1)}$, and so we can efficiently sample a stochastic gradient, and assign it to $g(t)$. If it returns "False", then w.h.p, $\gamma_t\le \alpha^{\Omega(1)}$. They key idea here is that, as we will show, $\gamma_t$ being small is actually an "easy" case, and simply choosing $g(t)=(A(t)x_t-y_t)x_t^\top$ suffices to make progress towards $A_*$. Observe that here $g(t)$ is a stochastic gradient of $\widetilde{f}_t(A)= -\mathbb{E}_{y\sim \mN\left(A_*x_t,I,S_t\right)}\Big{[}\ell_{\R^n}(Ax_t;y)\Big{]}$. In other words, the Test is a "switch" between $\nabla f_t$ and $\nabla \widetilde{f}_t$.


\SetKwInput{KwInput}{Input}              
\SetKwInput{KwOutput}{Output}  
\begin{algorithm}[H] 
 \KwInput{$A_0\in \R^{n\times d}$, PSD matrix $\Sigma_0 \in \R^{d\times d}$, data $(x_{t_i},y_{t_i})_{i=1}^N$. \\
 $\eta=(2/\alpha)^{c_\eta}$  \text{ \ \ \ \ \  \  \ \ \ \  \ \ \ \ \ \ \ \ \ \ \ \ \  \ \ \ \ \ \ \ \ \ \ \ \ \ \ \ \ \  \ \ \ \ \ \ \ \ \ \ \ \ \ \ \ \ \ \ \ \ \ \  \  \ \ \ $\rhd  c_\eta \ge 0$ is a large constant.} }
  $\mK=\{A\in \R^{n\times d}:\ \|A-A_0\|_{\Sigma_0}\le 1\}$\\
  $A_1=A_0$\\
  \For{$i=1$ to $N$}{
    $ g_i=\text{SwitchGrad}(A_ix_{t_i},x_{t_i},y_{t_i},S_{t_i})$\\
    $\Sigma_i=\Si_{i-1}+x_{t_i}x_{t_i}^\top$\\
    $\widetilde{A}_{i+1}=A_i-\eta \cdot g_i\Si_i^{-1}$\\
    $A_{i+1}=\argmin_{A\in \mK}\|A-\widetilde{A}_{i+1}\|_{\Si_i}$
  }
  \Return $\whA=A_{N+1}$
  \caption{Stochastic Online Newton with Switching Gradients}
  \label{alg:newton}
\end{algorithm}
\SetAlgoNoLine
\SetKwInput{KwInput}{Input}              
\SetKwInput{KwOutput}{Output}  
\begin{algorithm}[H] 
  \textbf{Input:} $\mu,x,y, S$\\
  $z=\mu$\\
  \If {\text{Test}$(\mu, S)$}{
    Sample $z'\sim \mN(\mu,I,S)$ via rejection sampling using the membership oracle $M_S$.\\
    $z=z'$
    }
    \Return $g=(z-y)x^T$
  \caption{\text{SwitchGrad}}
  \label{alg:switchgrad}
\end{algorithm}
\SetAlgoNoLine
\SetKwInput{KwInput}{Input}              
\SetKwInput{KwOutput}{Output}  
\begin{algorithm}[H] 
  \textbf{Input:} $\mu, S$. \\
  $\gamma=(\alpha/2)^{c_\gamma}$, $\ k=\frac{4}{\gamma}\log{T}.$\ \  \  \ \ \ \  \ \ \ \ \ \ \ \ \ \ \ \ \  \ \ \ \ \ \ \ \ \ \ \ \ \ \ \ \ \  \ \ \ \ \ \ \ \ \ \ \ \ \ \ \ \ \ \ \  \text{$\rhd  c_\gamma \ge 0$ is a large constant.} 
  \\
  Sample $\xi_1,\dots,\xi_k \simiid N(\mu,I)$\\
  $p=\frac{1}{k}\sum_{j=1}^k \cha\{\xi_j\in S\}$\\
  \Return $(p\geq 2\gamma)$
  \caption{\text{Test}}
  \label{alg:test}
\end{algorithm}

\section{Proof of Theorem \ref{thm:time-series}}\label{subs:prf1}
In this section, we give an overview of the proof of Theorem \ref{thm:time-series}, with an emphasis on the novel technical components. Our goal is to convey the key ideas, and so at some steps we are slightly informal. We provide the formal and detailed proof in Appendix \ref{appdx:prf:thm-time-series}.
\paragraph{The Generic Bound.}We first prove a bound on $\left\|\whA-A_* \right\|_{\Sigma}^2$ (remember that $\Sigma=\Sigma_N$ in Algorithm \ref{alg:newton}), which holds for any sequence $g_1,g_2,\dots,g_N$. This bound will serve as a guideline for choosing the $g_i$'s.
\begin{lemma}\label{lem:generic:new}
Independently of how $g_i$'s are chosen, 
\begin{align}\label{eq:generic:new}
    \left\|\whA-A_*\right\|_{\Sigma}^2\leq 1- \sum_{i=1}^N \left(2\eta \langle g_i, A_i -\sA \rangle - \left\|(A_i-\sA)x_{t_i}\right\|^2 \right) + \eta^2\sum_{i=1}^N\trace\left(g_i \Si_i^{-1}g_i^\top\right).
\end{align}
\end{lemma}
The proof of the lemma is along the lines of the analysis of ONS in \cite{hazan2007logarithmic}, and we provide it in Appendix \ref{appdx:prf-generic}. Let $E_1$ be the first sum in \ref{eq:generic:new}, and $E_2$ the second. We will show that for our choice of $g_i$'s, $E_1$ is (almost) non-negative, and $E_2$ is not too large (both in expectation). In the main text, we assume that $L=0$, since the extension for general $L$ is straightforward.
\subsection*{Bound on $E_1$.}
Fix a time $t$, condition on $\mF_t$, and suppose that $t=t_i$, for some $i$. To ease notation, we define $A(t)= A_i,\ g(t)=g_i,\ \mu_t^*= \sA x_t,\ \mu_t= A(t)x_t$, and $V_t= 2\eta \langle g(t), A(t) -\sA \rangle - \left\|(A(t)-\sA)x_{t}\right\|^2$. To bound $\Ex[E_1]$, we show (roughly) that $ \mathbb{E} \big[V_t\ \big|\ \mF_t, t=t_i\big]\ge 0$. Observe that this immediately implies $\Ex[E_1]\ge 0$. We consider two cases, based on what the Test function returns.
\paragraph{Large Survival Probability.} Suppose at time $t$, Test returns "True". We call this event $\mT_t$, and here we assume that Test returning "True" implies $\mN(\mu_t,I;S_t)\ge \gamma$ (this is correct w.h.p.).\footnote{The parameter $\gamma$ is defined in Algorithm \ref{alg:test}. Also, in Appendix \ref{appdx:prf:thm-time-series}, where we give the details for handling the low-probability events, we show that $\Ex[E_1]\ge \Delta$, for some $\Delta>0$, but small.} 
Furthermore, $z_t\sim \mN(\mu_t,I,S_t)$, and $g(t)=(z_t-y_t)x_t^T$. Also, since $y_t\sim \mN(\mu_t^*,I,S_t)$, we have
\begin{align}\label{eq:rephrasing:new}
     \mathbb{E} \big[V_t\ \big|\ \mF_t, t=t_i, \mT_t \big]= 2\eta \langle \nu_t - \nu_{t}^*, \mu_t-\mus_t \rangle - \|\mu_t-\mu_t^*\|^2,
\end{align}
where $\nu_t=\mathbb{E}_{z\sim \mN(\mu_t,I,S_t)}[z]$ and $\nu_t^*=\mathbb{E}_{y\sim \mN(\mu_t^*,I,S_t)}[y]$. Now, note that since $t=t_i$ and $L=0$, we have $\mN(\mu_t^*,I;S_t)\ge \alpha\ge\gamma$. The key component of our proof is the following lemma, which we prove here in detail.
\begin{lemma}\label{lem:mvt:new}
Let $\mL_S(\mu;\mus)=-\mathbb{E}_{y\sim \mN(\mu^*,I,S)}\big[\ell_S(\mu;y)\big]$, where  $\mu, \mu^* \in \mathbb{R}^n$ and $S\subseteq \R^n$. Suppose that $\mN(\mu, I; S),\ \mN(\mus, I; S)\ge \gamma$. Then, there exists an absolute constant $c>0$ such that
\begin{align}\label{eq:mvt-ineq:new}
 \big{\langle}\nabla_\mu \mL_{S}(\mu;\mu^*), \mu-\mu^*\big{\rangle} \geq \left(\gamma/2\right)^c \cdot
 \|\mu-\mu^*\|^2. 
\end{align}
\end{lemma}
Lemma \ref{lem:mvt:new} implies that, given large enough $c_\eta$,\footnote{Remember that $\eta=(2/\alpha)^{c_\eta}$.} the expectation in (\ref{eq:rephrasing:new}) is non-negative. Indeed, note that $\nabla_\mu L_{S}(\mu;\mu^*)= \nu -\nu^*, $ where $\nu=\mathbb{E}_{z\sim \mN(\mu,I,S)}[z]$ and $\nu^* = \mathbb{E}_{y\sim \mN(\mu^*,I,S)}[z]$, so 
\begin{align*}
\mathbb{E} \big[V_t\ \big|\ \mF_t, t=t_i, \mT_t \big]\ge(2\eta (\gamma/2)^c-1) \cdot \|\mu_t-\mu_t^*\|^2 &=\left(2\left(2/\alpha\right)^{c_\eta} (\alpha/4)^{c\cdot c_\gamma}-1\right)\cdot \|\mu_t-\mu_t^*\|^2 \\
& =\left(\frac{2^{c_\eta +1-2c\cdot c_\gamma}}{a^{c_\eta-c\cdot c_\gamma}}-1\right) \cdot \|\mu_t-\mu_t^*\|^2\ge 0,
\end{align*}
for any constant $c_\eta \ge 2c\cdot c_\gamma$. We now prove Lemma \ref{lem:mvt:new}.
\\
\begin{proof}
Let $s(\gamma)\coloneqq \sqrt{2\log(1/\gamma)}+1$. We will need two technical claims, which hold for any $\mu,\mu^*\in \R^n$ and $S\subseteq \R^{n}$. 
\begin{claim}\label{cl:means:new}
Let $\nu=\mathbb{E}_{z\sim \mN(\mu,I,S)}[z]$. If $\mN(\mu,I;S)\ge \gamma$, then $\|\nu-\mu\|\le s(\gamma)$.
\end{claim}
\begin{claim}\label{cl:cov:new}
Suppose $\mN(\mu^*,I;S)\ge \gamma$, and for some $\widetilde{\mu}$ we have  $\|\widetilde{\mu}-\mu^*\|\le c\cdot s(\gamma)$. Then, 
\begin{align*}
\text{Cov}_{z\sim \mN(\widetilde{\mu},I,S)}[z,z]\succcurlyeq (\gamma/2)^{\poly(c)}\cdot I.
\end{align*}
\end{claim}
Both Claims are mainly from \cite{daskalakis2018efficient}. Claim \ref{cl:means:new} is Lemma 6 in that paper. Claim \ref{cl:cov:new} is not explicitly stated, but can be derived by slightly adapting their proof (see Appendix \ref{appdx:prf:cov}).  We consider two cases:
\\
\\
 \textbf{Case 1:}   $\|\mu-\mu^*\|\ge 4 s(\gamma)$. Then,
\begin{align*}
    \langle \nu -\nu^*, \mu-\mu^* \rangle&= \|\mu-\mus\|^2 + \langle \nu-\mu, \mu-\mu^* \rangle +  \langle \mu^*-\nu^*, \mu-\mu^* \rangle  \notag \\
    & \geq \|\mu-\mus\|^2 -\|\nu-\mu\|\cdot \| \mu-\mu^*\| - \|\nu^*-\mu^*\|\cdot \|\mu-\mu^*\|
\end{align*}
Since $\mN(\mu, I; S),\ \mN(\mus, I; S)\ge \gamma$, Claim \ref{cl:means:new} implies $\|\nu-\mu\|,\|\nu^*-\mu^*\|\le s(\gamma)$.
Being in Case 1,
\begin{align*}
     \langle \nu -\nu^*, \mu-\mu^* \rangle \geq \|\mu-\mus\|\cdot \big(\|\mu-\mus\|- 2s(\gamma)\big)\geq \|\mu-\mu^*\|^2/2.
\end{align*}
\textbf{Case 2:} $\|\mu-\mu^*\|< 4 s(\gamma)$. To ease notation, fix $\mu^*, S$, and let $\mL(\mu)=\mL_{S}(\mu;\mu^*)$. From fundamental theorem of calculus, 
\begin{align*}
    \nabla \mL(\mu) -\nabla \mL(\mu^*)=\int_{0}^1 \nabla^2\mL\left(\mu(\theta)\right)d\theta \cdot (\mu-\mus),
\end{align*}
where $\mu(\theta)\coloneqq \mu^*+\theta(\mu-\mu^*)$. Since $\nabla L(\mu^*)=0$,
\begin{align}\label{eq:int-hes:new}
     \big{\langle}\nabla \mL(\mu), \mu-\mu^*\big{\rangle}= \int_{0}^1\Big{\langle} \nabla^2\mL\left(\mu(\theta)\right)  (\mu-\mus), \mu-\mu^*\Big{\rangle}d\theta.
\end{align}
For $\theta\in (0,1)$, $\|\mu(\theta)-\mu^*\|=\theta\|\mu-\mu^*\|\le O\left(s(\gamma)\right)$. Using $\mN(\mu^*,I,S)\ge \gamma$ and Claim \ref{cl:cov:new},
\begin{align*}
    \nabla^2\mL\left(\mu(\theta)\right)= \text{Cov}_{z\sim \mN\left(\mu(\theta),I,S\right)}[z,z] \succcurlyeq (\gamma/2)^{O(1)}\cdot I.
\end{align*}
Using \ref{eq:int-hes:new}, we finish the proof. 
\end{proof}
\paragraph{Remark.} We want to highlight an important qualitative difference between Lemma \ref{lem:mvt:new} and all recent works in CTS. Suppose $\gamma=\Omega(1)$ and let $v=\frac{\mu-\mu^*}{\|\mu-\mu^*\|}$. From the argument in Case 2, \ref{eq:mvt-ineq:new} can be rephrased as 
\begin{align}\label{eq:varlb}
    \Ex_{\widetilde{\mu}\sim[\mu^*,\mu]}\Big[\big{\langle}\text{Cov}_{z\sim \mN\left(\widetilde{\mu},I,S\right)}[z,z]\cdot v, v \big{\rangle}\Big] \ge \Omega(1),
\end{align}
where $\widetilde{\mu}$ is distributed uniformly on the segment $[\mu,\mu^*]$. In other words, the average variance in the $v$-direction is $\Omega(1)$. This is an anti-concentration bound. As we mentioned, proving anti-concentration bounds for truncated normals $\mN(\wt{\mu},I,S)$ is a core component in all recent works in CTS, where this task is reduced to lower bounding $\mN(\wt{\mu},I;S)$. However, inequality (\ref{eq:varlb}) cannot be proven with this methodology. We illustrate this with an insightful example (Figure \ref{fig:para1}). In the example, we consider $\mu,\mu^*\in \R^d$ and $\|\mu-\mu^*\|=\sqrt{d}$, because the guaranteed bound in the LDS-case for $\|\mu_t-\mu_t^*\|$ will be $\Theta(\sqrt{d})$. Let $\mu^*=0$, $\mu=\sqrt{d}\cdot e_1$ ($e_1$ is the standard-basis vector), $S=\{x\in \R^{d}:  x_1\in (-\infty,0]\cup [\sqrt{d},+\infty)\}$. Observe that the conditions of Lemma \ref{lem:mvt:new} are satisfied with $\gamma=1/2$. However, for most $\wt{\mu}\in[\mu^*,\mu]$, the mass $\mN(\wt{\mu},I;S)$ is exponentially small (in $d$). Consider now $\wt{\mu}$ exactly in the middle of $\mu^*,\mu$. Even though $\mN(\wt{\mu},I;S)=\exp(-\Omega(d))$, the variance in the $e_1$ direction is $\Omega(d)$, due to symmetry. It can actually be shown that precisely this $\wt{\mu}$ and small perturbations of it make the average variance $\Omega (1)$ in (\ref{eq:varlb}). Note that it is necessary for both corners ($\mu$ and $\mu^*$) to have high survival probability, e.g., in Figure \ref{fig:para2} (for the same distance-scales) it can be shown that the variance is $O(1/d)$. 
\begin{figure}[!tbp]
 \centering
 \begin{minipage}[b]{0.45\textwidth}
    \includegraphics[width=\textwidth]{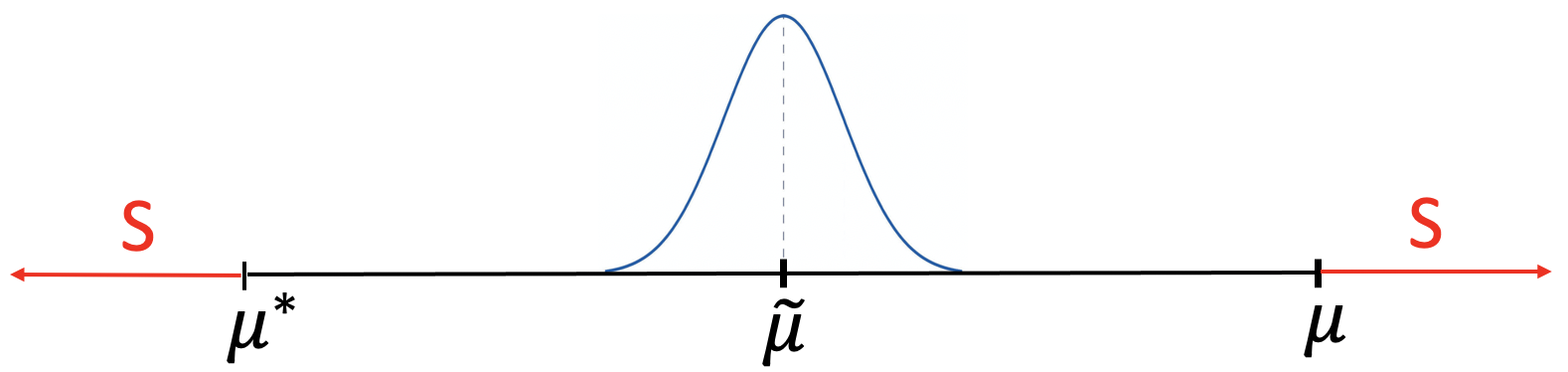}
    \caption{$\Omega(d)$ variance} 
    \label{fig:para1}
 \end{minipage}
 \hfill
 \begin{minipage}[b]{0.42\textwidth}
    \includegraphics[width=\textwidth]{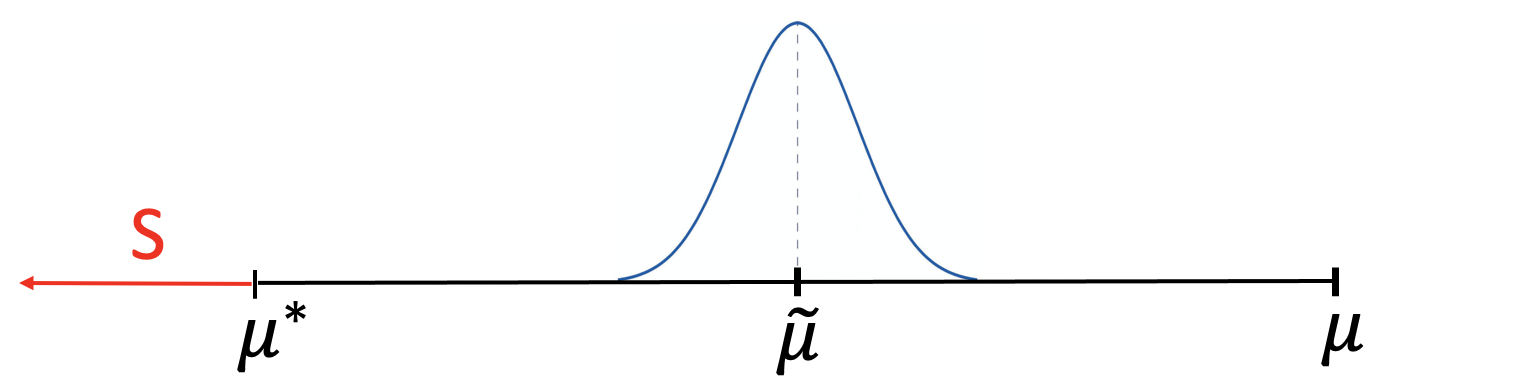}
    \caption{$\O(1/d)$ variance}
    \label{fig:para2}
 \end{minipage}
\end{figure}\label{fig:para}
\paragraph{Small Survival Probability.} Suppose that at time $t$, Test returns "False" ($\neg \mT_t$). Suppose also that this implies $\mN(\mu_t,I;S_t)\le 4\gamma$ (again this is w.h.p). We will show that $ \mathbb{E} \big[V_t\ \big|\ \mF_t, t=t_i, \neg \mT_t \big]\ge 0$. Observe that given $\neg \mT_t$, we have $g(t)=(\mu_t -y_t)x_t^\top$, and so
\begin{align}\label{eq:Vt-not-T}
     \mathbb{E} \big[V_t\ \big|\ \mF_t, t=t_i, \neg \mT_t \big]= 2\eta \langle \mu_t - \nu_{t}^*, \mu_t-\mus_t \rangle - \|\mu_t-\mu_t^*\|^2.
\end{align}
Again, since $t=t_i$ and $L=0$, we have $\mN(\mu_t^*,I;S_t)\ge \alpha$, so by Claim \ref{cl:means:new}, $\|\nu_t^*-\mu_t^*\|\leq s(\alpha)$. We need one more technical claim.
\begin{claim}\label{cl:prob-lb:new}
If $\mN(\mus,I;S)\ge \alpha$, and $\|\mu-\mus\|\le  c\cdot s(\alpha)$, then $\mN(\mu,I;S)\ge (\alpha/2)^{\poly(c)}$.
\end{claim}
Again, Claim \ref{cl:prob-lb:new} is proven in Appendix \ref{appdx:prf:prob-lb}, by slightly adapting the analysis of \cite{daskalakis2018efficient}. Now, we claim that $\|\mu_t-\mu_t^*\|> 2s(\alpha)$, provided that $c_\gamma=O(1)$ is sufficiently large. Indeed, suppose $\|\mu_t-\mu_t^*\| \le 2 s(\alpha)$. Combining with $\mN(\mu_t^*,I;S_t)\ge \alpha$ and Claim \ref{cl:prob-lb:new}, we get $4(\alpha/2)^{c_\gamma} =4\gamma \ge \mN(\mu_t,I; S_t)\ge (\alpha/2)^{O(1)}$, which is a contradiction for large enough $c_\gamma=O(1)$. Using that $\|\nu_t^*-\mu_t^*\|\le s(\alpha)$, we get
\begin{align*}
    \langle \mu_t - \nu_{t}^*, \mu_t-\mus_t \rangle &= \|\mu_t -\mus_t\|^2 + \langle \mu_t^* - \nu_{t}^*, \mu_t-\mus_t \rangle \ge \|\mu_t -\mus_t\|^2-\|\mu_t^* - \nu_{t}^*\| \cdot \|\mu_t -\mus_t\| \\
    &\ge \|\mu_t -\mus_t\|\cdot \big( \|\mu_t -\mus_t\| -s(\alpha) \big) \ge  \|\mu_t-\mu_t^*\|^2 /2.
\end{align*}
Since $\eta\ge 1$, Equation \ref{eq:Vt-not-T} implies $ \mathbb{E} \big[V_t\ \big|\ \mF_t, t=t_i, \neg \mT_t \big]\ge 0$. Thus, $\Ex[V_t\ |\ \mF_t,t=t_i]\ge 0$, and $\Ex[E_1]\ge 0$.

\subsection*{Bound on $E_2$}
Using the assumption that $\mK$ is $(R,\omega)$-accurate, we prove that 
\begin{align}\label{bound:e2}
    \Ex\big[E_2\big] \leq \tO(1)\cdot  d\left(d+R^2+R_w^2\right).
\end{align}
Given this bound, we have
\begin{align*}
    \Ex\left[\left\|\whA-A_*\right\|_{\Sigma}^2\right]&\leq  1-\Ex[E_1]+\eta^2 \Ex[E_2]\leq  \tO(1)\cdot \poly(1/\alpha)\cdot d\left(d+R^2+R_w^2\right),
\end{align*}
which finishes the proof of Theorem \ref{thm:time-series} (for $L=0$). Now, we illustrate the main steps for proving (\ref{bound:e2}) (the complete proof is in Appendix \ref{appdx:prf:var}). Let $z_{t_i}$ be the value of the variable $z$ of the SwitchGrad function, at the end of its execution, during iteration $i$. So, $g_i=(z_{t_i}-y_{t_i})x_{t_i}^\top$, and thus $\trace(g_i\Sigma_i^{-1}g_i^\top)=\|z_{t_i}-y_{t_i}\|^2\cdot x_{t_i}^\top\Si_i ^{-1}x_{t_i}$. Summing over $i$,
\begin{align}\label{eq:e2-max}
    E_2=\sum_{i=1}^N \|z_{t_i}-y_{t_i}\|^2\cdot x_{t_i}^\top\Si_i ^{-1}x_{t_i}\le \max_{i}\left\{\|z_{t_i}-y_{t_i}\|^2\right\}\cdot  \sum_{i=1}^N x_{t_i}^\top\Si_i ^{-1}x_{t_i}.
\end{align}
Using the recurrence $\Sigma_i=\Sigma_{i-1} +x_{t_i }x_{t_i}^\top$, we can show that  $\sum_{i=1}^N x_{t_i}^\top\Si_i ^{-1}x_{t_i}\le d\left(\log{|\Sigma_N|^{1/d}}+\log{(1/\omega )}\right)$, where $|\Sigma_N|$ denotes the determinant. The proof of this is given in Appendix \ref{appdx:prf:var} (Claim \ref{cl:det-pot}), and is based on a linear-algebraic technique of \cite{hazan2007logarithmic}. Combining with $E_2$, and applying Cauchy-Schwarz,
\begin{align*}
    \Ex[E_2]=O(d)\cdot \sqrt{\Ex \left[ \max_{i}
  \|z_{t_i}-y_{t_i}\|^4\right]} \cdot \sqrt{\Ex\left[\log^2{|\Sigma_N|^{1/d}}\right]+\tO(1)}.
\end{align*}
We use that $\Ex[\|x_t\|^2]\le R_x^2$, to control $\Ex\left[\log^2{|\Sigma_N|^{1/d}}\right]$. Specifically, in Appendix \ref{appdx:prf:var} (Claim \ref{cl:det:cs}), we show that $\Ex\left[\log^2{|\Sigma_N|^{1/d}}\right]\le  \log^2\left(e+T\cdot R_x^2\right)=\tO(1)$. The last step is to prove that 
\begin{align}\label{eq:max-main-1}
    \Ex \left[ \max_{i}
  \|z_{t_i}-y_{t_i}\|^4\right]\leq \tO(1)\cdot (d^2+R^4+R_w^4).
\end{align}
This is shown in Appendix \ref{appdx:prf:var} (Claim \ref{cl:max}), and here we give a proof-sketch. Fix an $i$, and let $t=t_i$. We decompose $\|z_t-y_t\|$:
\begin{align}\label{decomp:main}
   \|z_t-y_t\|=\|z_t-\mu_t+\mu_t -\mu_t^* +\mu_t^* -y_t\|&\le \|z_t-\mu_t\| +\|\mu_t -\mu_t^*\| +\|\mu_t^* -y_t\| \notag \\
   &= \|z_t-\mu_t\|\cdot \cha\{\mT_t\} +\|\mu_t -\mu_t^*\| +\|\mu_t^* -y_t\| \notag \\
   &= \|z_t-\mu_t\|\cdot \cha\{\mT_t\} +\|\mu_t -\mu_t^*\| +\|w_t\|
\end{align}
where for the second-to-last step we use that on $\neg \mT_t$, we have $z_t=\mu_t$. Now, $\|w_t\|$ is controlled using the assumption that it is $R_w^2$-subgaussian. Now,
\begin{align*}
\|\mu_t -\mu_t^*\|=  \left\|(A(t)-\sA)x_t\right\|&=  \left\|(A(t)-\sA)\Sigma_0^{1/2} \Sigma_0^{-1/2} x_t\right\| \\
&\le \left\|(A(t)-\sA)\Sigma_0^{1/2}  \right\|_2 \cdot  \left\| \Sigma_0^{-1/2} x_t \right\| \\
& \le  \left\|A(t)-\sA  \right\|_{\Sigma_0} \cdot  \left\| \Sigma_0^{-1/2} x_t \right\|\\
& \le  2 \left\| \Sigma_0^{-1/2} x_t \right\|,
\end{align*}
where for the last step we used that both $A(t)$ and $\sA$ belong to $\mK$, which is guaranteed by the projection step in Algorithm \ref{alg:newton}, and the assumption that $\mK$ is $(R,\omega)$-accurate. Now, $\left\| \Sigma_0^{-1/2} x_t \right\|$ is controlled using the assumption that it is $R^2$-subgaussian (again because of the $(R,\omega)$-accuracy of $\mK$). The final step is to control $ \|z_t-\mu_t\|\cdot \cha\{\mT_t\}$.  As we did previously, we assume that Test returning "True" implies $\mN(\mu_t,I;S_t)\ge \gamma$. Now, it can be shown (Appendix \ref{appdx:prob}, Claims \ref{appdx:prob1}, \ref{appdx:prob2}) that $\mN(\mu_t,I;S_t)\ge \gamma$ implies that $\|z_t-\mu_t\|$ is $O\left(
\log\left(\frac{1}{2\gamma}\right) d \log{d}\right)$-subgaussian. Now, the bound (\ref{eq:max-main-1}) follows from standard properties of subgaussian random variables.

%% file: projection-set.tex
\section{The Initial Ellipsoid and Proof of Theorem \ref{thm:main}}\label{sec:projection}
Here, we provide the overall algorithm (Algorithm \ref{alg:overall}), and an outline of the proof of Theorem \ref{thm:main}.
Let $\mI_0=\{t\in [T]: (x_t,x_{t+1})\in \mP_0\}$, and $\mI_1=\{t\in [T]: (x_t,x_{t+1})\in \mP_1\}$. 
\\
\SetKwInput{KwInput}{Input}              
\SetKwInput{KwOutput}{Output}  
\begin{algorithm}[H] 
  $A_0=\argmin_{A\in \R^{d\times d}} \sum_{t\in \mI_0} \|x_{t+1}-Ax_t\|^2$\\
  $\Sigma_0= \frac{1}{s\cdot|\mI_0|}\cdot \sum_{t\in \mI_0}x_tx_t^\top$, \ where $s= c_s \left(\sqrt{\log{(1/\alpha)}}+1\right)$.  \ \ \ \ \ \ \  \text{$\rhd  c_s \ge 0$ is a large constant.} \\
  Get $\whA$ by running SON-SG (Algorithm \ref{alg:newton}) with $A_0,\Sigma_0$, and dataset $\mP_1$.\\
  \Return $\whA$
  \caption{SON-SG for censored linear dynamics}
  \label{alg:overall}
\end{algorithm}
\SetAlgoNoLine
The least-squares method when all $x_t$'s are observed was analyzed in \cite{simchowitz2018learning}. Here, we observe only a subset the $x_t$'s. In Appendix \ref{appdx:prf:init}, we generalize the "small-ball" technique of \cite{simchowitz2018learning}, to prove the following theorem.
\begin{theorem}\label{thm:init:new}
There exist $c_{1},c_{2}, c_{3}=\poly\left(\frac{1}{\alpha\beta}\right)$, such that if $T\ge  \widetilde{\Theta}(c_{1})\cdot \left(d^2+\frac{d}{1-\rho(\sA)}+dL\right)$, then with probability $1-o(1)$, the ellipsoid $\mK=\{A\in \R^{d\times d}: \|A-A_0\|_{\Sigma_0}\le 1\}$ is $(R,\omega)$-accurate with $R=c_{2}\sqrt{d}$, $\omega= 1/c_{2}$, and also $ \frac{1}{|\mI_1|}\sum_{t\in \mI_1}x_tx_t^T \succcurlyeq \frac{1}{c_{3}} \ga_T$.
\end{theorem}
 We now prove Theorem \ref{thm:main} using Theorems \ref{thm:time-series} and \ref{thm:init:new}.
\begin{proof}
We use the time-series model with $y_t=x_{t+1}$, $S_t=\mS_{t+1}$, $o_t=\mathbbm{1}(t\in \mI_1)$. Thus, $B(\alpha)\subseteq \mB(\alpha)$, so  $\Ex[|B(\alpha)|]\le L$. Let $\mE_0$ be the event that all guarantees in Theorem \ref{thm:init:new} hold, so $\Pr[\neg \mE_0]\le o(1)$. As we mentioned, $\|w_t\|$ is $O(d)$-subgaussian. Also, in Appendix \ref{cl:Rx}, we show that $\Ex[\|x_t\|^2]$ is polynomial in all parameters. Let $\Sigma=\frac{1}{|\mI_1|}\sum_{t\in \mI_1}x_tx_t^T$.
By conditioning on $\mE_0$ and applying Theorem \ref{thm:time-series}, using Markov's inequality in (\ref{eq:error-bound}), we have that with probability at least $1-\delta$, 
\begin{align}\label{eq:markov-final}
    \left\|\whA-\sA\right\|_{\Sigma}^2&\le \frac{\tO(C)}{\delta N}\cdot\left(d(d+R^2+R_w^2)+L(d+R^2+R_w^2)\right) \le \tO(C)\cdot \frac{d^2+dL}{\delta N}.
\end{align}
where $C=\poly\left(\frac{1}{\alpha\beta}\right)$. Now, on $\mE_0$, $\Sigma\succcurlyeq \ga_T/\poly\left(\frac{1}{\alpha\beta}\right)$. From Proposition \ref{prop:largeN}, $\Pr\left[N<\frac{\alpha\beta T}{4} \right]\le o(1)$. Thus, by choosing small enough $\delta=\Omega(1)$, we are done.
\end{proof}

\section{Conclusion and Future Work}
In this paper, we developed the first computationally and statistically efficient algorithm for learning censored LDSs. We believe that there are several interesting directions for future work:
\begin{itemize}
    \item The LDSs studied here ($x_{t+1}=\sA x_t+w_t$) often appear in the literature under the name "vector autoregressive models of order 1" (VAR(1)). By including more lagged states in the equation, we get the VAR$(p)$ model: $x_{t+1}=A_{1}x_t+A_{2}x_{t-1}+\dots +A_{p}x_{t-p+1}+w_t$, where $p\ge 1$. Can we efficiently learn censored VAR$(p)$ models? This case (in 1-d) was considered by \cite{zeger1986regression} and \cite{park2007censored}. The challenge here is that the extension of our approach would use only tuples of $p$ consecutive observations, which will be very rare for large $p$. Even for $p=1$, can we leverage the isolated observations?
    \item In many systems, we do not observe $x_t$ directly, but only $C x_t$ for a short matrix $C\in \R^{d_0\times d}$, where $d_0 \le d$. These are called partially-observed LDSs, and their censored version was studied by \cite{allik2015tobit}. Understanding the sample complexity here would require overcoming the challenge sketched above, i.e., going beyond long tuples of consecutive observations.
    \item Finally, censoring often occurs in coordinate level, i.e., only a subset of the coordinates of $x_t$ is observed. This can be thought of as a ReLU nonlinearity, and shows up in applications with sensor saturations (\cite{yang2009set}). Can we efficiently learn the system under coordinate-level censoring?
\end{itemize}

%% file: appendix.tex
\section{Additional Notation}\label{appdx:not}
For a random variable $X$, we denote by $\|X\|_{\psi_2}$ the subgaussian norm of $X$, i.e.,
\begin{align*}
   \|X\|_{\psi_2}=\inf\{\sigma>0:\ \Ex[\exp(X^2/\sigma^2)]\le 2\}. 
\end{align*}
For a random vector $X$, the sugaussian norm is $ \|X\|_{\psi_2}=\sup_{v\in \mS^{d-1}}\|v^\top    X\|_{\psi_2}$. For two reals $a,b$ we denote $a\lor b$ and $a\land b$ the maximum and minimum respectively.
\section{One-dimensional Example}\label{appdx:1d}
Note that $x_{t+1}=\sum_{s=0}^{t-1}a_*^s w_{t-s}$, so $x_{t+1}\sim \mN(0,\sum_{s=0}^{t-1}a_*^{2s})$. We will need the following claim.
\begin{claim}
Under Assumption \ref{ass:observable} with $\beta=\Omega(1)$, we must have $\lambda\le O\left(\frac{1}{\sqrt{1-a_*}}\right)$.
\end{claim}
\begin{proof}
Fix a time $t$. Since $x_{t+1}\sim \mN(0,\sum_{s=0}^{t-1}a_*^{2s})$, we have $\text{Var}[x_{t+1}]=O\left(\frac{1}{1-a_*}\right)$, and so from Gaussian concentration: $\Pr\left[x_{t+1}\ge \kappa/\sqrt{1-a_*}\ \right]\le \exp(-\Omega(\kappa^2))$. By choosing large enough $\kappa=\Theta(\sqrt{\log{(1/\beta)}})$, we get that $\Pr\left[x_{t+1}\ge \kappa/\sqrt{1-a_*}\ \right]\le \beta/4$. Thus, 
\begin{align}\label{eq:1-d}
    \Ex\left[\sum_{t=1}^T \cha\{x_{t+1}\ge \kappa/(1-a_*)\}\right]\le \beta T/4.
\end{align}
Now, suppose that $\lambda< \kappa/\sqrt{1-a_*}$. Then, $\Ex[|\mO|]=\Ex\left[\sum_{t=1}^T \cha\{x_{t+1}\ge \lambda\}\right]\le \beta T/4$, and so from Markov's inequality, $\Pr\left[|\mO|\ge \beta T/2 \right]\le 1/2$, which contradicts Assumption \ref{ass:observable}. Thus, $\lambda< \kappa/\sqrt{1-a_*}$, and since $\kappa=\Theta(\sqrt{\log{(1/\beta)}})$, and $\beta=\Omega(1)$, we are done.
\end{proof}
The claim implies that $\lambda \le O\left(\frac{1}{1-a_*}\right)$. Let $C$ denote the constant on this bound, i.e., $\lambda \le \frac{C}{1-a_*}$. We consider two cases. First, if $x_t\le \frac{C}{1-a_*}$, then $a_* x_t +C \ge x_t$, i.e, if $x_t\in \mS$, then $x_{t+1}\in \mS$ with probability $\exp(-O(C^2))=\Omega(1)$. If $x_t>\frac{C}{1-a_*}$, then $a_*x_t+C\ge \frac{C}{1-a_*}\ge \lambda$. Again, since $C=O(1)$, we are done.
\section{Proof of Proposition \ref{prop:largeN}}\label{appdx:pairs}
We will need some definitions. We define the probabilities $\alpha_t^*\coloneqq\mN(\sA x_t,I;\mS_{t+1})$, and the events $\mE_t^\alpha\coloneqq\{\alpha_t^*\ge \alpha\}$, and $\mO_t\coloneqq\{x_t\in \mS_t\}$. Also, consider the process
\begin{align*}
    B_t\coloneqq\Big(\mathds{1}\{\mO_{t+1}\}-\alpha \Big)\cdot\mathds{1}\{\mO_t\land \mE_t^\alpha\},
\end{align*}
and the filtration $G_t\coloneqq \sigma(w_0,w_1,\dots,w_t,\mS_1,\mS_2,\dots, \mS_{t+2})$. Then,
\begin{align}\label{eq:prop-b1}
    \Ex[B_t\ |\ G_{t-1}]&=\Ex\bigg[\Big(\mathds{1}\{\mO_{t+1}\}-\alpha \Big)\cdot\mathds{1}\{\mO_t\land \mE_t^\alpha\}\ \bigg|\ G_{t-1}\bigg] \notag\\
    &=\Pr\Big[\mO_{t+1}\ \Big|\ G_{t-1}, \mO_t\land \mE_t^\alpha\Big]-\alpha\ge 0,
\end{align}
where in the last step we use that given $G_{t-1}$ and $\mO_t\land \mE_t^\alpha$, we have $x_{t+1}\sim \mN(\sA x_t,I)$, and $\alpha_t^*\ge \alpha$. Now, let $M_t\coloneqq B_t-\Ex[B_t|G_{t-1}]$, and observe that $M_t$ is $G_t$-measurable, and that 
$\Ex[M_t\ |\ G_{t-1}]=0$. Since $-1\le B_t \le 1$, we have $-2\le M_t\le 2$. Thus, from Azuma-Hoeffding inequality, 
\begin{align}\label{eq:prop-b2}
    \Pr\left[\sum_{t=1}^TM_t \le -\frac{\alpha \beta T }{4}\right]\le \exp{\left(-\frac{\alpha^2\beta^2 }{128} T\right)}=o(1).
\end{align}
Moreover,
\begin{align}\label{eq:prop-b3}
   \sum_{t=1}^T M_t& = \sum_{t=1}^T \Big(B_t-\Ex[B_t|G_{t-1}]\Big)\le^{(\ref{eq:prop-b1})}\ \sum_{t=1}^T B_t \notag \\
   &=\sum_{t=1}^T \cha\{\mO_{t+1}\land \mO_t\land \mE_t^\alpha\}-\alpha \sum_{t=1}^T\cha\{\mO_t\land \mE_t^\alpha\} \notag\\
   &\le |\mP|-\alpha \sum_{t=1}^T\cha\{\mO_t\land \mE_t^\alpha\}.
\end{align}
From Assumptions \ref{ass:observable} and \ref{ass:survival-prob}, we have
\begin{align}\label{eq:prop-b4}
   \Pr\left[\sum_{t=1}^T\cha\{\mO_t\land \mE_t^\alpha\}<\beta T-L\right]\leq o(1).
\end{align}
Finally, using the assumed lower bound on $T$, we get $\beta T-L\ge \frac{3\beta}{4}T$. The proposition follows after combining this with 
\ref{eq:prop-b2},\ref{eq:prop-b3} and \ref{eq:prop-b4}. $\blacksquare $
\section{Detailed Proof of Theorem \ref{thm:time-series}}\label{appdx:prf:thm-time-series}

\subsection{The Generic bound}\label{appdx:prf-generic}
We will show the following lemma
\begin{lemma}\label{lem:generic}
Independently of how $g_i$'s are chosen, 
\begin{align}\label{eq:generic}
    \left\|\whA-A_*\right\|_{\Sigma}^2\leq 1- \underbrace{\sum_{i=1}^N \left(2\eta \langle g_i, A_i \widetilde{A}_{i+1}-\sA \rangle - \left\|(A_i-\sA)x_{t_i}\right\|^2 \right)}_{E_1} + \eta^2\underbrace{\sum_{i=1}^N\trace\left(g_i \Si_i^{-1}g_i^\top\right)}_{E_2}.
\end{align}
\end{lemma}
\begin{proof}
For all $i$, 
\begin{align}\label{eq:regr-b1}
   \widetilde{A}_{i+1}-\sA=A_i-\sA -\eta g_i \Si_i^{-1}. 
\end{align}
Multiplying with $\Si_i$,
\begin{align}\label{eq:regr-b2}
    (\widetilde{A}_{i+1}-\sA)\Si_i=(A_i-\sA)\Si_i-\eta g_i. 
\end{align}
Multiplying \ref{eq:regr-b2} with the transpose of \ref{eq:regr-b1},
\begin{align*}
    (\widetilde{A}_{i+1}-\sA)\Si_i(\widetilde{A}_{i+1}-\sA)^\top &=\left((A_i-\sA)\Si_i-\eta g_i\right)\left(A_i-\sA -\eta g_i \Si_i^{-1}\right)^\top  
\end{align*}
Expanding and taking trace, 
\begin{align*}
    \|\widetilde{A}_{i+1}-\sA\|_{\Si_i}^2=\|A_i-\sA\|_{\Si_i}^2 -2\eta \langle g_i, A_i -\sA \rangle +\eta^2 \ \trace \left(g_i\Si_i^{-1} g_i^\top\right)
\end{align*}
Since $\mK$ is $(R,\omega)$-accurate, $\sA\in \mK$, and by convexity of $\mK$, $\|A_{i+1}-\sA\|_{\Si_i}\leq \|\widetilde{A}_{i+1}-\sA\|_{\Si_i}$ (see \cite{hazan2019introduction} Theorem 2.1). Thus,
\begin{align*}
    \|A_{i+1}-\sA\|_{\Si_i}^2\le\|A_i-\sA\|_{\Si_i}^2 -2\eta \langle g_i, A_i -\sA \rangle +\eta^2 \ \trace \left(g_i\Si_i^{-1} g_i^T\right)
\end{align*}
Summing over all $i$,
\begin{align}\label{eq:regret-b3}
   0\le  \sum_{i=1}^N\left(\|A_i-\sA\|_{\Si_i}^2-\|A_{i+1}-\sA\|_{\Si_i}^2\right)-2\eta \sum_{i=1}^N  \langle g_i, A_i -\sA \rangle +\eta^2  \sum_{i=1}^N \trace \left(g_i \Si_i^{-1} g_i^\top\right)
\end{align}
We now focus on the first sum.
\begin{align}\label{eq:regret-b4}
    \sum_{i=1}^N\left(\|A_i-\sA\|_{\Si_i}^2-\|A_{i+1}-\sA\|_{\Si_i}^2\right)=\|A_1-\sA\|_{\Si_1}^2 +\sum_{i=2}^N \Big( \|A_i-\sA\|_{\Si_i}^2 &-\|A_{i}-\sA\|_{\Si_{i-1}}^2 \Big) \notag\\
    &- \|A_{N+1}-\sA\|_{\Si_N}^2.
\end{align}
Also, since $\Sigma_i=\Si_{i-1}+x_{t_i}x_{t_i}^T$,
\begin{align}\label{eq:regret-b5}
 \|A_i-\sA\|_{\Si_i}^2 -\|A_{i}-\sA\|_{\Si_{i-1}}^2 &=\trace\left((A_i-\sA)^\top(A_i-\sA)\Sigma_i\right)-  \trace\left((A_i-\sA)^\top(A_i-\sA)\Sigma_{i-1}\right) \notag\\
 &=\trace\left((A_i-\sA)^\top(A_i-\sA)x_{t_i}x_{t_i}^\top)\right) =\|(A_i-\sA)x_{t_i}\|^2.
\end{align}
Thus,
\begin{align}\label{eq:regret-b6}
    \sum_{i=1}^N\big(\|A_i-\sA\|_{\Si_i}^2-\|A_{i+1}-\sA &\|_{\Si_i}^2\big)=\|A_1-\sA\|_{\Si_1}^2  -\|A_{N+1}-\sA\|_{\Si_N}^2+\sum_{i=2}^N \|(A_i-\sA)x_{t_i}\|^2 \notag \\
    &=^{(\ref{eq:regret-b5})}\|A_1-\sA\|_{\Si_0}^2  -\|A_{N+1}-\sA\|_{\Si_N}^2+\sum_{i=1}^N \|(A_i-\sA)x_{t_i}\|^2. 
\end{align}
Combining \ref{eq:regret-b3} with \ref{eq:regret-b6},
\begin{align*}
     \|A_{N+1}-\sA\|_{\Si_N}^2\le \|A_1-\sA\|_{\Si_0}^2-\sum_{i=1}^N \left( 2\eta   \langle g_i, A_i -\sA \rangle -\|(A_i-\sA)x_{t_i}\|^2\right) +\eta^2  \sum_{i=1}^N \trace \left(g_i \Si_i^{-1} g_i^\top\right).
\end{align*}
Observe that $\|A_1-\sA\|_{\Si_0}^2=\|A_0-\sA\|_{\Si_0}^2\le 1$, since $A_1=A_0$ and $\sA\in \mK$. The facts that $\whA=A_{N+1}$ and $\Sigma=\Sigma_N$ finish the proof.    
\end{proof}
\subsection{Decomposition of $E_1$}\label{e1}
We start with some definitions. Fix a $t$, and let
\begin{align*}
    \mathcal{T}_t \coloneqq \Big{\{} o_t=1\ \text{and at step $t$, Test returns True} \Big{\}}.
\end{align*}
Also, let $O_t=\{o_t=1\}$. We now define some "good" events. Let $G_t\coloneqq \{ t\notin B\}$,
    \begin{align*}
        C_t\coloneqq \left\{ O_t\rightarrow \left\{\text{ at iteration $i(t)$, either Test returns True and $a_t\ge \gamma$, or Test returns False and $a_t\le 4\gamma$} \right\} \right\},
    \end{align*}
   \begin{align*}
  K_t \coloneqq \left\{ \|\Si_0^{-1/2}x_{t}\|\le R\sqrt{\log T}\right\}
    \end{align*}

$G_t\coloneqq \big\{\mN(\sA x_t,I;S_t)\ge \alpha \big {\}}$,
and observe that $|B(\alpha)|=\sum_{t=1}^T\cha\{\neg G_t \land o_t=1\}$. The next event is about the "correctness" of the Test. Let 
\begin{align*}
    C_t\coloneqq \Big{\{} o_t=1\ \text{and at that step, either Test returns True and $\gamma_t\ge \gamma$, or Test returns False and $\gamma_t\le 4\gamma$} \Big{\}}, 
\end{align*}
where $a_t$ is defined in Section \ref{sec:algo}. We aggregate via $\mE_t \coloneqq G_t \land C_t \land K_t \land O_t$. Now, given $o_t=1$, let $V_t\coloneqq  2\eta \langle g(t), A(t) -\sA \rangle - \left\|(A(t)-\sA)x_{t}\right\|^2$, where $A(t)$ and $g(t)$ are defined in the main text. We decompose $E_1$ as 
\begin{align*}
    E_1=\sum_{t=1}^T V_t  \cha\{\mT_t\land \mE_t\}+\sum_{t=1}^T V_t  \cha\{\neg \mT_t\land \mE_t\} + \sum_{t=1}^T V_t  \cha\{\neg \mE_t \land O_t\}.
\end{align*}
Let $Q_1,Q_2,Q_3$ be the above three sums, respectively.

\subsection{Large Survival Probability}
In this step, we prove that $\Ex[V_t\ |\ \mT_t\land \mE_t]\ge 0$, which implies $\Ex[Q_1]\ge 0$. First of all, we condition on $\mF_t$ and $\mT_t\land \mE_t$. Then, letting $\mu_t\coloneqq A(t)x_t$ and $\mu_t^*\coloneqq \sA x_t$, we get $g(t)=(z_t-y_t)x_t^\top$, where $y_t\sim \mN(\mu_t^*,I,S_t)$ and $z_t\sim \mN(\mu_t,I,S_t)$. So,
\begin{align*}
\Ex\Big[V_t\ \Big{|}\ \mF_t,\ \mT_t\land \mE_t  \Big{]} =2\eta \langle \nu_t - \nu_{t}^*, \mu_t-\mus_t \rangle - \|\mu_t-\mu_t^*\|^2,
\end{align*}
where $\nu_t=\mathbb{E}_{z\sim \mN(\mu_t,I,S_t)}[z]$ and $\nu_t^*=\mathbb{E}_{y\sim \mN(\mu_t^*,I,S_t)}[y]$. Also, note that on $\mT_t\land \mE_t$,
\begin{align*}
     \mN(\mu_t,I;S_t),\ \mN(\mu_t^*,I;S_t)\geq \gamma
\end{align*}
Now, using Lemma \ref{lem:mvt:new} of the main text (proven in detail there), we see that for large enough constant $c_\eta$\footnote{Remember that $\eta=(2/\alpha)^{c_{\eta}}$.}, we get $\Ex[V_t\ |\ \mF_t,\ \mT_t\land \mE_t ] \geq 0$. Indeed, note that $\nabla_\mu L_{S}(\mu;\mu^*)= \nu -\nu^*, $ where $\nu=\mathbb{E}_{z\sim \mN(\mu,I,S)}[z]$ and $\nu^* = \mathbb{E}_{y\sim \mN(\mu^*,I,S)}[z]$. Thus,
\begin{align*}
\Ex\Big[V_t\ \Big{|}\ \mF_t,\ \mT_t\land \mE_t  \Big{]} \ge(2\eta (\gamma/2)^c-1) \cdot \|\mu_t-\mu_t^*\|^2 &=\left(2\left(2/\alpha\right)^{c_\eta} (\alpha/4)^{c\cdot c_\gamma}-1\right)\cdot \|\mu_t-\mu_t^*\|^2 \\
& =\left(\frac{2^{c_\eta +1-2c\cdot c_\gamma}}{a^{c_\eta-c\cdot c_\gamma}}-1\right) \cdot \|\mu_t-\mu_t^*\|^2\ge 0,
\end{align*}
for any constant $c_\eta \ge 2c\cdot c_\gamma$.
\subsection{Small Survival Probability}
Here, we prove that $\Ex[V_t\ |\ \neg \mT_t\land \mE_t]\ge 0$, which implies $\Ex[Q_2]\ge 0$. Given $\mF_t $ and $\neg \mT_t\land \mE_t$, we have $g(t)=(\mu_t-y_t)x_t^T$, where $y_t\sim \mN(\mu_t^*,I,S_t)$. Thus,
\begin{align}\label{eq:Vt-not-T}
\Ex\Big[V_t\ \Big{|}\ \mF_t,\ \neg \mT_t\land \mE_t  \Big{]} =2\eta \langle \mu_t - \nu_{t}^*, \mu_t-\mus_t \rangle - \|\mu_t-\mu_t^*\|^2.
\end{align}
Now, since $\mE_t$ implies $C_t$, we have that given $\neg \mT_t\land \mE_t$, the probability $\mN(\mu_t,I;S_t)\le 4\gamma$. Moreover, $\mE_t$ implies $G_t$, i.e., $\mN(\mu_t^*,I;S_t)\ge \alpha$, which from Claim \ref{cl:means:new} implies $\|\nu_t^*-\mu_t^*\|\leq s(\alpha)$. We claim that $\|\mu_t-\mu_t^*\|> 2s(\alpha)$, provided that $c_\gamma=O(1)$ is sufficiently large. Indeed, suppose $\|\mu_t-\mu_t^*\| \le 2 s(\alpha)$. Combining with $\mN(\mu_t^*,I;S_t)\ge \alpha$ and Claim \ref{cl:prob-lb:new}, we get 
\begin{align*}
   4(\alpha/2)^{c_\gamma} =4\gamma \ge \mN(\mu_t,I, S_t)\geq (\alpha/2)^{O(1)},
\end{align*}
which is a contradiction for large enough $c_\gamma=O(1)$. Using that $\|\nu_t^*-\mu_t^*\|\le s(\alpha)$,
\begin{align*}
    \langle \mu_t - \nu_{t}^*, \mu_t-\mus_t \rangle &= \|\mu_t -\mus_t\|^2 + \langle \mu_t^* - \nu_{t}^*, \mu_t-\mus_t \rangle \ge \|\mu_t -\mus_t\|^2-\|\mu_t^* - \nu_{t}^*\| \cdot \|\mu_t -\mus_t\| \\
    &\ge \|\mu_t -\mus_t\|\cdot \big( \|\mu_t -\mus_t\| -s(\alpha) \big) \ge  \frac{\|\mu_t-\mu_t^*\|^2}{2}.
\end{align*}
Since $\eta\ge 1$, equation \ref{eq:Vt-not-T} implies $\Ex[V_t\ |\ \mF_t, \neg \mT\land \mE_t]\ge 0$.
\subsection*{Handling Bad Events and Gradient Variance}
For $Q_3$, we show that 
\begin{align*}
   \Ex\big[| Q_3| \big]\leq \tO(1)\cdot \poly(1/\alpha)\cdot  \left(d+R^2+R_w^2\right)\cdot(L+1).
\end{align*}
This is proven using that bad events have small probabilities, relative to $V_t$'s. This is a purely technical proof, and we provide it in Appendix \ref{appdx:prf:Q3-bound}.
\\
\\
The last step is to bound $E_2=\sum_{i=1}^N\trace\left(g_i \Si_i^{-1}g_i^T\right).$ This is often called "the gradient variance" in optimization. We show that 
\begin{align*}
    \Ex\big[E_2\big] \leq \tO(1)\cdot  d(d+R^2+R_w^2).
\end{align*}
The proof uses the volume-based potential-function argument used in \cite{hazan2007logarithmic}, which also appears in linear bandit theory (\cite{lattimore2020bandit}). We provide the proof in Appendix \ref{appdx:prf:var}. Overall,
\begin{align*}
    \Ex\left[\left\|\whA-A_*\right\|_{\Sigma}^2\right]&\leq  1-\Ex[Q_1]-\Ex[Q_2]-\Ex[Q_3]+\eta^2 \Ex[E_2] \\
    &\leq \tO(1)\cdot \poly(1/\alpha)\cdot \left(d+R^2+R_w^2\right)\cdot(L+1) + \tO(1)\cdot \poly(1/\alpha)\cdot d(d+R^2+R_w^2)\\
    & \leq D+LD',
\end{align*}
where $D=dD'$, and $D'=\tO(1)\cdot \poly(1/\alpha) \cdot (d+R^2+R_w^2) $.

\section{Bound on $\Ex[|Q_3|]$}\label{appdx:prf:Q3-bound}
\begin{lemma}\label{lem:q3}
$\mathbb{E}[|Q_3|]\le\tO(1)\cdot \poly(1/\alpha)\cdot (d+R^2+R_w^2)\cdot (L+1)$.
\end{lemma}
We now prove the above lemma. We start with some definitions (and reminders). If at time $t$ we have $o_t=1$, then we define the following:
\begin{itemize}
     \item Let $i(t)$ be the iteration $i$, which corresponds to time $t$.
    \item Let $A(t)\coloneqq A_{i(t)}$, and $\mu_t\coloneqq A(t)x_t$. Also, let $\mu_t^*\coloneqq \sA x_t$.
    \item Let $\gamma_t\coloneqq\mN(\mu_t,I;S_t)$.
    \item Let $z_t$ be the $z$ at the end of SwitchGrad function, when called at iteration $i(t)$. 
    \item If additionally, Test at iteration $i(t)$ returns True ($\mT_t$), then let $z_t'$ be the $z'$ of the SwitchGrad function, at iteration $i(t)$. Observe that on $\mT_t$, $z_t=z_t'$. 
\end{itemize}
Now, observe that given $o_t=1$,
\begin{align*}
    V_t =2\eta \langle z_t-y_t  , \mu_t-\mu_t^*\rangle - \|\mu_t-\mu_t^*\|^2.
\end{align*}
All the above quantities are defined on $O_t=\{o_t=1\}$ (except for $z_t'$ which is defined on $\mT_t\subseteq O_t$). We extend all these, by defining them to be zero outside of $O_t$ ($\mT_t$ for $z_t'$). Furthermore, remember that $\mE_t = G_t \land C_t \land K_t \land O_t$, where the events (along with $\mT_t$) are defined in \ref{e1}.
Moreover, in Appendix \ref{appdx:prob} (Claim \ref{cl:chernoff}), we show that $\Pr[\neg C_t]\le 1/T^2$. About $K_t$, since $\mK$ is $(R,\omega)$-accurate, $\Pr[\neg K_t]\le O(1/T)$.
We first decompose $|Q_3|$.
\begin{claim}
$|Q_3|\le \eta Q_4+\eta Q_5$, where
\begin{align}
    Q_4= \sumt (8\dmus^2+2\|w_t\|^2)\cha\{\neg \mE_t \land O_t\},\ \ \ \ \text{and}\ \ \ \ Q_5=  \sumt \|z_t'-\mu_t\|^2\cha\{\neg \mE_t \land \mT_t\}.
\end{align}
\end{claim}
\begin{proof}
Let $w_t\coloneqq y_t-\mu_t^*$ (the noise). We have $|Q_3|\le  \sum_{t=1}^T |V_t|\cha\{\neg \mE_t \land O_t\}$, and
\begin{align*}
    |V_t|&  \le 2\eta \left| \langle z_t'-y_t , \mu_t-\mu_t^*\rangle \right| \cdot \cha\{\mT_t\}+2\eta  \left|\langle \mu_t-y_t , \mu_t-\mu_t^*\rangle \right|  +\dmus^2 \\
    &\le 2\eta \| z_t'-y_t\| \cdot \| \mu_t-\mu_t^*\|\cdot \cha\{\mT_t\}+ 2\eta \| \mu_t-y_t\| \cdot \| \mu_t-\mu_t^*\| +\eta \dmus^2\\
    & \le 2\eta \| z_t'-\mu_t\| \cdot \| \mu_t-\mu_t^*\|\cdot \cha\{\mT_t\}+ 4\eta \| \mu_t-y_t\| \cdot \| \mu_t-\mu_t^*\| +\eta \dmus^2\\
    & \le 2\eta \| z_t'-\mu_t\| \cdot \| \mu_t-\mu_t^*\|\cdot \cha\{\mT_t\}+ 4\eta \| \mu_t-\mu_t^*\|^2+4\eta \| w_t\| \cdot \| \mu_t-\mu_t^*\| +\eta \dmus^2\\
    &\le \eta \| z_t'-\mu_t\|^2\cdot \cha\{\mT_t\}+6\eta \| \mu_t-\mu_t^*\|^2 + 2\eta \| \mu_t-\mu_t^*\|^2+2\eta \| w_t\|^2 \\
    &= \eta \| z_t'-\mu_t\|^2\cdot \cha\{\mT_t\}+8\eta \| \mu_t-\mu_t^*\|^2 +2\eta \| w_t\|^2
\end{align*}
where we have used that $\eta\ge 1$, and that for all $a,b\in \R$, $2ab\le a^2+b^2$. Since $\mT_t$ implies $ O_t$,
\begin{align*}
    |Q_3|\le \eta \sumt  \| z_t'-\mu_t\|^2\cdot \cha\{\neg \mE_t\land \mT_t\}+ \eta \sumt (8 \| \mu_t-\mu_t^*\|^2 +2 \| w_t\|^2)\cdot \cha\{\neg \mE_t\land O_t\}.
\end{align*}
\end{proof}
We will prove the following claims.
\begin{claim}\label{appdx:cl:q4}
$\Ex[Q_4]\le \tO(R^2+R_w^2)\cdot(L+1)$.
\end{claim}
\begin{claim}\label{appdx:cl:q5}
$\Ex[Q_5]\le \tO(d)\cdot(L+1)$.
\end{claim}
Observe that these two claims complete the bound on $\Ex[|Q_3|]$. We now prove the claims.

\subsection{Proof of Claim \ref{appdx:cl:q4}}
We will bound $\Ex[Q_4]= \Ex\left[\sumt (8\dmus^2+2\|w_t\|^2)\cha\{\neg \mE_t \land O_t\}\right]$ by proving the following two claims.
\begin{claim}\label{appdx:prop:R}
$\sumt \Ex\left[\dmus^2 \cha\{\neg \mE_t \land O_t\}\right] \le \tO(1)\cdot R^2(L+1)$
\end{claim}
\begin{claim}\label{appdx:prop:Rw}
$\sumt \Ex\left[\|w_t\|^2 \cha\{\neg \mE_t \land O_t\} \right]\le \tO(1)\cdot R_w^2(L+1)$
\end{claim}
\subsubsection{Proofs of Claims \ref{appdx:prop:R}, \ref{appdx:prop:Rw}}
On $O_t$, $A(t)=A_{i(t)}$ and so $A(t),\sA \in \mK$, in other words $\|A_0-\sA\|_{\Si_0}\le 1$ and $\|A_0-A(t)\|_{\Si_0}\le 1$. From triangle inequality, $\|A(t)-\sA\|_{\Si_0}\le 2$, and so
\begin{align}\label{eq:acc:conseq}
    \dmus &=\|(A(t)-\sA)x_t\|=\|(A(t)-\sA)\Si_0^{1/2}\Si_0^{-1/2} x_t\| \notag\\
    &\le \|(A(t)-\sA)\Si_0^{1/2}\|_2\cdot \|\Si_0^{-1/2} x_t\| \notag \\
    &\le\|(A(t)-\sA)\Si_0^{1/2}\|_{\text{F}} \cdot \|\Si_0^{-1/2} x_t\|\le 2 \| x_t\|_{\Si^{-1}}.
\end{align} 
Thus, since $\mK$ is $(R,\omega)$-accurate, $\dmus$ is $O(R^2)$-subgaussian, which implies
\begin{align}\label{eq:prop:R:1}
    \Pr\left[\dmus>R\sqrt{\log{(TR)}}\right] \le \frac{1}{T^2R^4}\ \ \ \text{and} \ \ \ \Ex\left[\dmus^4\right]\le O(R^4).
\end{align} 
Let $K_t'$ be the event in above probability.
\begin{align*}
   \Ex\left[ \dmus^2 \cha\{\neg \mE_t \land O_t\}\right]&= \Ex \left[\dmus^2 \cha\{\neg \mE_t \land O_t \land (\neg K_t')\}\right]+\Ex\left[\dmus^2 \cha\{\neg \mE_t \land O_t \land  K_t'\}\right]\\
    &\le^{(\ref{eq:prop:R:1})} \tO(R^2)\cdot \Pr[ \mE_t \land O_t] +\Ex[\dmus^2\cha\{ K_t'\}]\\
    &\le  \tO(R^2)\cdot (\Pr[\neg G_t]+\Pr[\neg K_t]+\Pr[\neg C_t])+\sqrt{\Ex[\dmus^4]}\cdot \sqrt{\Pr[K_t']} \\
    &\le^{(\ref{eq:prop:R:1})}\tO(R^2)\cdot (\Pr[\neg G_t]+1/T)+O(R^2)\cdot \frac{1}{TR^2}\\
    &\le \tO(R^2)\cdot (\Pr[\neg G_t]+1/T).
\end{align*}
Summing over $t$ and using that $\sumt \Pr[\neg G_t]\le L$, we finish the proof. For Claim \ref{appdx:prop:Rw}, observe that since $\|w_t\|$ is $R_w^2$ subgaussian, we can just use identical steps as above.  $\blacksquare$

\subsection{Proof of Claim \ref{appdx:cl:q5}}

Let $\mE_{t,5}\coloneqq \neg \mE_t \land \mT_t$, and so $Q_5=\sumt \|z_t'-\mu_t\|^2\cha\{\mE_{t,5}\}$
\begin{claim}\label{cl:Q3:1}
Given $\mF_t$ and $\mE_{t,5}$, we have that
$\|z_t'-\mu_t\|_2$ is $\tO(d)\cdot (1\lor \log{(1/\gamma_t)})$-subgaussian.
\end{claim}
\begin{proof}
Given $\mF_t$ and $\mE_{t,5}$, we have $z_t'\sim \prt $. In Appendix \ref{appdx:prob} (Claim \ref{appdx:prob1}) we show that this implies $\|z_t'-\mu_t\|_{\psi_2}^2\le O(1\lor \gamma_t)$. Claim \ref{appdx:prob2} in Appendix \ref{appdx:prob} finishes the proof.
\end{proof}
The claim above implies (see \cite{vershynin2018high}) 
\begin{align}\label{eq:z_prime}
    \Ex[\|z_t'-\mu_t\|_2^2 \ |\ \mF_t, \mE_{t,5}]\le \tO(d)\cdot (1\lor \log{(1/\gamma_t)}).
\end{align}
Thus, 
\begin{align}\label{eq:q3:b0}
   \Ex\left[ \|z_t'-\mu_t\|^2\cha\{\mE_{t,5}\}\right]&= \Ex\left[\Ex\left[ \|z_t'-\mu_t\|^2\cha\{\mE_{t,5}\} \ |\ \mF_t\right]\right] \notag \\
    &= \Ex\left[\Ex\left[ \|z_t'-\mu_t\|^2\cha\{\mE_{t,5}\} \ |\ \mF_t\right]\right] \notag \\
   & = \Ex\left[\Ex\left[ \|z_t'-\mu_t\|^2 \ |\ \mF_t, \mE_{t,5}\right] \cdot \Pr[\mE_{t,5}\ |\ \mF_t]\right] \notag \\
   & \le  \tO(d)\cdot \Ex\left[  (1\lor \log{(1/\gamma_t)})\cdot \Pr[\mE_{t,5}\ |\ \mF_t]\right] \notag \\
   &= \tO(d)\cdot \Ex\left[ \Ex\left[ (1\lor \log{(1/\gamma_t)})\cdot \cha\{\mE_{t,5}\} \ |\ \mF_t \right]\right] \notag\\
   &=\tO(d)\cdot \Ex\left[ (1\lor \log{(1/\gamma_t)})\cdot \cha\{\mE_{t,5}\} \right],
\end{align}
where the inequality above is justified by (\ref{eq:z_prime}). So,
\begin{align}\label{eq:Q3:b2}
   \Ex\left[ \|z_t'-\mu_t\|^2\cha\{\mE_{t,5}\}\right]
   &\le \tO(d)\cdot \Big( \Pr\left[ \mE_{t,5} \right]+\Ex\left[  \log{(1/\gamma_t)}\cdot \cha\{\mE_{t,5}\} \right] \Big)
\end{align}
We will first bound 
\begin{align}\label{eq:log-bound-2}
\Ex[  \log{(1/\gamma_t)}\cdot &\cha\{\mE_{t,5}\} ]=\Ex\left[  \log{(1/\gamma_t)}\cdot \cha\{C_t \land \mE_{t,5}\} \right]+\Ex\left[  \log{(1/\gamma_t)}\cdot \cha\{\neg C_t \land \mE_{t,5}\} \right] \notag\\
&=\Ex\left[  \log{(1/\gamma_t)}\cdot \cha\{C_t \land \mT_t\}\cdot  \cha\{\neg(G_t \land K_t)\} \right]+\Ex\left[  \log{(1/\gamma_t)}\cdot \cha\{\neg C_t \land \mT_t\} \right],
\end{align}
where we used the definition of $\mE_{t,5}$. Now, observe that on $C_t\land \mT_t$, we have $\gamma_t\ge \gamma$. Thus,
\begin{align}\label{eq:Q3:b3}
   \Ex\left[  \log{(1/\gamma_t)}\cdot \cha\{C_t \land \mT_t\}\cdot  \cha\{\neg(G_t \land K_t)\} \right]& \le \log{(1/\gamma)}\cdot\Pr[\neg(G_t \land K_t)] \notag \\
   &\le \log{(1/\gamma)}\cdot (\Pr[\neg G_t]+\Pr[\neg K_t]).
\end{align}
We now bound $\Ex\left[  \log{(1/\gamma_t)}\cdot \cha\{\neg C_t \land \mT_t\} \right]$ via the following claim.
\begin{claim}\label{cl:Q3:log}
Fix $r=O(1)$. Then, $\Ex\left[  \left(\log{\frac{1}{\gamma_t}}\right)^r\cdot \cha\{\neg C_t \land \mT_t\} \right]\le \tO\left(1\right)\cdot \frac{\gamma^{-1}}{ T^2}$.
\end{claim}
We state the claim for general $r$, because later we will need it for $r=2$. Observe that for $r=1$, we get $\Ex\left[  \log{(1/\gamma_t)}\cdot \cha\{\neg C_t \land \mT_t\} \right]\le \tO\left(1\right)\cdot \frac{\gamma^{-1}}{ T^2}$. We prove the claim at the end of the subsection. Combining with (\ref{eq:log-bound-2}) and (\ref{eq:Q3:b3}), we get 
\begin{align*}
    \Ex[ \log{(1/\gamma_t)}\cdot \cha\{\mE_{t,5}\} ]\le \log{(1/\gamma)}\cdot (\Pr[\neg G_t]+\Pr[\neg K_t])+ \tO\left(1\right)\cdot \frac{\gamma^{-1}}{ T^2},
\end{align*}
so from (\ref{eq:Q3:b2}),
\begin{align*}
   \Ex\left[ \|z_t'-\mu_t\|^2\cha\{\mE_{t,5}\}\right]&\le \tO(d)\cdot \left(\Pr[\mE_{t,5}]+ \Pr[\neg G_t]+\Pr[\neg K_t]+\frac{\gamma^{-1}}{ T^2}\right) \\
   &\le  \tO(d)\cdot \left(\Pr[\mE_{t,5}]+ \Pr[\neg G_t]+\Pr[\neg K_t]+\frac{\gamma^{-1}}{ T^2}\right) 
\end{align*}
Observe that $\Pr[\mE_{t,5}]\le \Pr[\neg C_t]+\Pr[\neg K_t]+\Pr[\neg G_t]$. Also, $\sumt Pr[\neg G_t]=\Ex[|B(\alpha)|]\le L$, and $\Pr[\neg K_t]\le O(1/T)$ (since $\mK$ is $(R,\omega)$-accurate), and $\Pr[\neg C_t]\le 1/T^2$ (Appendix \ref{appdx:prob}, Claim \ref{cl:chernoff}). Thus, from (\ref{eq:q3:b0}),
\begin{align*}
    \Ex[Q_5]\le \tO(d)\cdot \sumt  \left(\Pr[\neg G_t]+  1/T+\frac{\gamma^{-1}}{ T^2}\right) \le \tO(d)\cdot(L+1),
\end{align*}
using that $T\ge \poly(1/\alpha)$ and $1/\gamma=\poly(1/\alpha)$.
\\

We now give the proof of Claim \ref{cl:Q3:log}.                                                                             
\begin{proof}
Let $P_t=\{\gamma_t \ge 1/T^2\}$. Then,
\begin{align}\label{eq:log1}
    \Ex\Big[ \left( \log{\frac{1}{\gamma_t}}\right)^r \cdot \cha\{\neg C_t \land \mT_t\} \Big]= \Ex\Big[  \left( \log{\frac{1}{\gamma_t}}\right)^r\cdot & \cha\{\neg C_t \land \mT_t 
    \land P_t \} \Big]\notag \\
    &+\Ex\left[  \left( \log{\frac{1}{\gamma_t}}\right)^r\cdot \cha\left\{\neg C_t \land \mT_t 
    \land (\neg P_t)\right\} \right].
\end{align}
Since $ \mT_t 
    \land (\neg P_t)$ implies $\neg C_t$, we have
\begin{align*}
    \Ex\left[  \left( \log{\frac{1}{\gamma_t}}\right)^r\cdot \cha\{\neg C_t \land \mT_t 
    \land (\neg P_t)\} \right]& =\Ex\left[  \left( \log{\frac{1}{\gamma_t}}\right)^r\cdot \cha\{ \mT_t 
    \land (\neg P_t)\} \right] \\
    & \le \Ex\left[  \left( \log{\frac{1}{\gamma_t}}\right)^r\cdot \cha\{ \mT_t\}
  \  \bigg {|}\ \neg P_t\right] \\
  & = \Ex\left[  \Ex\left[\left( \log{\frac{1}{\gamma_t}}\right)^r\cdot \cha\{ \mT_t\}
  \  \bigg {|}\ \gamma_t, \neg P_t\right]   \  \bigg {|}\ \neg P_t \right] \\
  &=  \Ex\left[ \left( \log{\frac{1}{\gamma_t}}\right)^r \cdot \Pr\left[ \mT_t
  \  \big {|}\ \gamma_t, \neg P_t\right]   \  \bigg {|}\ \neg P_t \right].
\end{align*}
Combining with (\ref{eq:log1}), we get
\begin{align*}
    \Ex\Big[ \left( \log{\frac{1}{\gamma_t}}\right)^r \cdot \cha\left\{\neg C_t \land \mT_t\right\} \Big]\le & \left( \log{T^2}\right)^r\cdot  \Ex\Big[\cha\{\neg C_t \land \mT_t \} \Big] \\ 
    &+\Ex\left[ \left( \log{\frac{1}{\gamma_t}}\right)^r \cdot \Pr\left[ \mT_t
  \  \big {|}\ \gamma_t, \neg P_t\right] \  \bigg {|}\ \neg P_t \right].
\end{align*}
From Appendix \ref{appdx:prob} (Claim \ref{cl:chernoff}), we get 
$\Ex\left[\cha\{\neg C_t \land \mT_t \} \right]\le \Ex\left[\cha\{\neg C_t \land O_t \} \right]\le 1/T^2$. Also, given $\gamma_t$, the probability of $\mT_t$ is at most $1-(1-\gamma_t)^k$ (if no $\xi_j$ hits $S_t$, then the Test definitely returns False), where $k$ is defined in Algorithm \ref{alg:test}. From Bernoulli's inequality, $1-(1-\gamma_t)^k\le \gamma_t k$, and so
\begin{align*}
    \Ex\left[ \left( \log{\frac{1}{\gamma_t}}\right)^r \cdot \Pr\left[ \mT_t
  \  \big {|}\ \gamma_t, \neg P_t\right]  \  \bigg {|}\ \neg P_t \right]& \le  k\cdot \Ex\left[  \left( \log{\frac{1}{\gamma_t}}\right)^r\cdot \gamma_t\ \Big| \ \gamma_t<1/T^2 \right] \\
    &\le k\cdot \frac{\left(\log{\left(T^2\right)}\right)^r} {T^2} \le \tO(1)\cdot \frac{\gamma^{-1}}{T^2},
\end{align*}
where we used that $r=O(1)$ and that $T$ is lower-bounded by a large constant.
\end{proof}

\section{Bound on Gradient Variance}\label{appdx:prf:var}
\begin{lemma}\label{lem:e2}
$\mathbb{E}[E_2]\le\tO(1)\cdot d\cdot (d+R^2+R_w^2)$.
\end{lemma}
In this section, we prove Lemma \ref{lem:e2}.
\begin{align*}
   E_2&=\sum_{i=1}^N \trace(g_i\Si_i^{-1} g_i^T)\le   \sum_{i=1}^N \|z_{t_i}-y_{t_i}\|^2 \cdot (x_{t_i}^T\Si_i^{-1} x_{t_i})\le  \max_{i}
   \left\{ \|z_{t_i}-y_{t_i}\|^2 \right \} \cdot\sum_{i=1}^N  x_{t_i}^T\Si_i^{-1} x_{t_i}
\end{align*}
\begin{claim}\label{cl:det-pot}
$\sum_{i=1}^N  x_{t_i}^T\Si_i^{-1} x_{t_i} \le d\left(\log{|\Sigma_N|^{1/d}}+\log{(1/\omega )}\right)$, where $|\Sigma_N|$ denotes the determinant.
\end{claim}
This Claim is (implicitly) proven in \cite{hazan2007logarithmic}. We provide a proof for completeness.
\begin{proof}
We will use the following inequality: if $A,B$ are PSD matrices, then $\langle A^{-1},A-B\rangle \le \log{\frac{|A|}{|B|}}$. For a proof of this see \cite{hazan2019introduction} Lemma 4.6. 
\begin{align*}
   \sum_{i=1}^N  x_{t_i}^\top\Si_i^{-1} x_{t_i}&= \sum_{i=1}^N \langle \Si_i^{-1}, x_{t_i} x_{t_i}^\top \rangle  =  \sum_{i=1}^N \langle \Si_i^{-1}, \Si_i-\Si_{i-1}\rangle \le \sum_{i=1}^N  \log{\frac{|\Si_i|}{|\Si_{i-1}|}}=\log{\frac{|\Si_N|}{|\Si_{0}|}}.
\end{align*}
The Assumption that $\mK$ is $(R,\omega)$-accurate, and so $\Si_0\succcurlyeq \omega \cdot I$ completes the proof.
\end{proof}
By Cauchy-Schwarz,
\begin{align*}
  \Ex [E_2]\le  O(d)\cdot \sqrt{\Ex \left[ \max_{i}
  \|z_{t_i}-y_{t_i}\|^4\right]} \cdot \sqrt{\Ex\left[\log^2{|\Sigma_N|^{1/d}}\right]+\tO(1)}.
\end{align*}
The following claims complete the proof.
\begin{claim}\label{cl:det:cs}
$\Ex\left[\log^2{|\Sigma_N|^{1/d}}\right] \le  \log^2\left(e+T\cdot R_x^2\right)$.
\end{claim}
\begin{claim}\label{cl:max}
$\Ex \left[ \max_{i}
  \|z_{t_i}-y_{t_i}\|^4\right]\le \tO(1)\cdot (d^2+R^4+R_w^4)$
\end{claim}
We now prove these two claims.
\subsection{Proof of Claim \ref{cl:max}}
Let $M\coloneqq \max_{i}\|z_{t_i}-y_{t_i}\|$, and also let
\begin{enumerate}
  \item  $M_1\coloneqq \max_{t} \left\{\|z_{t}'-\mu_{t}\| \cdot \cha\{C_t\land  \mT_t\} \right\}$
  \item $M_2\coloneqq \max_{t} \left\{\|z_{t}'-\mu_{t}\| \cdot \cha\{ \neg C_t \land \mT_t\} \right\}$
  \item $M_3\coloneqq \max_{t} \|{\Si_0}^{-1/2}x_t\|$
  \item $M_4\coloneqq \max_{t} \|w_t\|$
\end{enumerate}
\begin{proposition}
$M\le M_1+M_2+4M_3+2M_4$
\end{proposition}
\begin{proof}
$M_1=\max_{t} \left\{\|z_{t}-y_{t}\| \cdot \cha\{O_t\} \right\}$. Now,
\begin{align*}
    \|z_{t}-y_{t}\| \cdot \cha\{O_t\} &=\|z_{t}-y_{t}\| \cdot \cha\{\mT_t\}+\|z_{t}-y_{t}\| \cdot \cha\{\neg \mT_t\land O_t\} \\
    & = \|z_{t}'-y_{t}\| \cdot \cha\{\mT_t\}+\|\mu_{t}-y_{t}\| \cdot \cha\{\neg \mT_t\land O_t\} \\
    & \le \|z_{t}'-y_{t}\| \cdot \cha\{\mT_t\}+\|\mu_{t}-y_{t}\| \cdot \cha\{ O_t\} \\
    & \le \left( \|z_{t}'-\mu_{t}\|+ \|\mu_t-y_{t}\|\right) \cdot \cha\{\mT_t\}+\|\mu_{t}-y_{t}\| \cdot \cha\{ O_t\} \\
    &\le   \|z_{t}'-\mu_{t}\|\cdot \cha\{\mT_t\}+2\|\mu_{t}-y_{t}\| \cdot \cha\{ O_t\}\\
    &\le  \|z_{t}'-\mu_{t}\|\cdot \cha\{\mT_t\}+2\|\mu_{t}-\mu_{t}^*\| \cdot \cha\{ O_t\} +2\|w_t\|\\
    &\le \|z_{t}'-\mu_{t}\|\cdot \cha\{\mT_t\}+4\|\Sigma_0^{-1/2} x_t\| +2\|w_t\|\\
    &= \|z_{t}'-\mu_{t}\|\cdot \cha\{C_t\land \mT_t\}+\|z_{t}'-\mu_{t}\|\cdot \cha\{\neg C_t\land \mT_t\}+4\|\Sigma_0^{-1/2} x_t\| +2\|w_t\|,
\end{align*}
where in the second-to-last step we used (\ref{eq:acc:conseq}).
\end{proof}
We will prove the following claims, which finish the proof of Claim \ref{cl:max}.
\begin{claim}\label{prop:M1}
$\Ex[M_1^4]\le \tO(d^2)$.
\end{claim}
\begin{claim}\label{prop:M2}
$\Ex[M_2^4]\le \tO(1)\cdot \frac{d^2}{\gamma T^2}$.
\end{claim}
\begin{claim}\label{prop:M34}
$\Ex[M_3^4]\le \tO(R^4)$, and $\Ex[M_4^4]\le \tO(R_w^4)$.
\end{claim}
First, Claim \ref{prop:M34} follows immediately from (a) $\|w_t\|$ being $R_w^2$-subgaussian, (b) $\|\Sigma_0^{-1/2} x_t\|$ being $R^2$-subgaussian, and (c) Claim \ref{appdx:prob3-max} (Appendix \ref{appdx:prob}). We now prove the other two claims.
\subsubsection{Proof of Claim \ref{prop:M1}}
Let $\sigma^2\coloneqq (1\lor \log{(1/\gamma)})\cdot d \cdot \log{(2d)}$. We will show that for all $t$,  $\|z_{t}'-\mu_{t}\| \cdot \cha\{C_t\land  \mT_t\} $ is $O(\sigma^2)$-subgaussian. Given this, Claim \ref{appdx:prob3-max} (Appendix \ref{appdx:prob}) finishes the proof. Fix an $r>0$.
\begin{align*}
    \Pr\left[\|z_{t}'-\mu_{t}\| \cdot \cha\{C_t\land  \mT_t\}  \ge \sigma \cdot r\right]\le \Pr\left[\|z_{t}'-\mu_{t}\|   \ge \sigma \cdot r\ \big|\ C_t\land  \mT_t\right]
\end{align*}
Observe given $C_t\land  \mT_t$, we have $\gamma_t\ge \gamma$. Let $\sigma_t^2\coloneqq (1\lor \log{(1/\gamma_t)})\cdot d \cdot \log{(2d)} $.
\begin{align*}
    \Pr\left[\|z_{t}'-\mu_{t}\| \cdot \cha\{C_t\land  \mT_t\}  \ge \sigma \cdot r\right]&\le \Pr\left[\|z_{t}'-\mu_{t}\|   \ge \sigma_t \cdot r\ \big|\ C_t\land  \mT_t\right] \\
    &=\Ex\left[\Pr[\|z_{t}'-\mu_{t}\|   \ge \sigma_t \cdot r \ |\ \mF_t,\mT_t] \ |\  C_t \land \mT_t\right].
\end{align*}
But, given $\mF_t$ and $\mT_t$, we have $z_t'\sim \mN(\mu_t,I,S_t)$, and so from Claims \ref{appdx:prob1} and \ref{appdx:prob2} (Appendix \ref{appdx:prob}), $\|z_t'-\mu_t\|$ is $O(\sigma_t^2)$-subgaussian. Thus, $ \Pr\left[\|z_{t}'-\mu_{t}\| \cdot \cha\{C_t\land  \mT_t\}  \ge \sigma \cdot r\right]\le e^{-\Omega(r^2)}$.     $\ 
\blacksquare$
\subsubsection{Proof of Claim \ref{prop:M2}}
As previously, given $\mF_t$ and $\neg C_t\land \mT_t$, we have $z_t'\sim \mN(\mu_t,I,S_t)$, and so from Claims \ref{appdx:prob1} and \ref{appdx:prob2} (Appendix \ref{appdx:prob}), $\|z_t'-\mu_t\|$ is $\sigma^2_t$-subgaussian. Thus,
\begin{align*}
    \Ex \Big[\|z_t'-\mu_t\|^4 \cdot \cha\{\neg C_t \land \mT_t\} \Big]&=\Pr[\neg C_t\land \mT_t]\cdot  \Ex\bigg[\Ex\Big[\|z_t'-\mu_t\|^4\ \big|\ \mF_t,\neg C_t\land \mT_t\Big]\ \Big|\ \neg C_t\land \mT_t\bigg] \\
    &\le O(1)\cdot \Ex\left[\sigma^4_t \cdot \cha\{\neg C_t\land \mT_t\}\right]\\
    &\le  \tO(d^2)\cdot \Ex\left[\left(1\lor \log(1/\gamma_t)\right)^2 \cdot \cha\{\neg C_t\land \mT_t\}\right] \\
    &\le \tO(d^2)\cdot \left(\Pr[\neg C_t]+\Ex\left[\left(\log(1/\gamma_t)\right)^2 \cdot \cha\{\neg C_t\land \mT_t\}\right] \right)\\
    &\le \tO(d^2)\cdot(1/T^2+\gamma^{-1}/T^2),
\end{align*}
where in the last step we used Claim \ref{cl:chernoff} (Appendix \ref{appdx:prob}) and Claim \ref{cl:Q3:log}. $\blacksquare$
\subsection{Proof of Claim \ref{cl:det:cs}}
The function $\log^2(x)$ is concave for $x\ge e$. Thus, from Jensen's inequality,
\begin{align*}
    \Ex\left[\log^2|\Si_N|^{1/d}\right]\le \Ex\left[\log^2\left(e+|\Si_N|^{1/d}\right)\right]\le \log^2\left(e+\Ex\left[|\Si_N|^{1/d}\right]\right)
\end{align*}
Now, from AM-GM, $|\Si_N|^{1/d}\le \frac{1}{d}\cdot \trace(\Si_N)\le\sum_{t=1}^T \|x_t\|^2$. Thus,
\begin{align*}
    \Ex\left[\log^2|\Si_N|^{1/d}\right]\le \log^2\left(e+T\cdot R_x^2\right).
\end{align*}

\section{Proof of Theorem \ref{thm:init:new}}\label{appdx:prf:init}
For all $\mI \subseteq [T]$, we define $X_{\mI}\in \R^{d\times |\mI|}$ be the matrix whose columns are the $\{x_t\}_{t\in \mI}$, placed in temporal order. Also, let $X\coloneqq X_{[T]}$, and $\Gamma_t :=  \sum_{s=0}^{t-1} (\sA^{s})(\sA^s)^{\top}$. The following lemma (a) upper and lower bounds the "size" of the covariates, (b) certifies that our algorithms will use enough data and that $|\mB(\alpha)|$ is small, and (c) bounds $\|w_t\|$.
\begin{lemma}\label{lem:Eg}
Let $\mI'_0\coloneqq \mI_0 \setminus \mB(\alpha)$, and $C_{\alpha,\beta}=\poly(1/\alpha,1/\beta)$ (sufficiently large). Also, let
\begin{align*}
    \mE_{g,1}\coloneqq \left\{|\mB(\alpha)|\le L\right\}\land \left\{|\mI_0|,|\mI_1|\ge \alpha\beta T/4\right\},\ \ \text{and} 
\end{align*}
\begin{align*}
    \mE_{g,2}\coloneqq \left\{\frac{1}{T}XX^\top \preccurlyeq  \frac{d}{\delta}\Gamma_T\right\}\land \left\{\frac{1}{|\mI_0|}X_{\mI_0'}X_{\mI_0'}^\top  \succcurlyeq \frac{1}{C_{\alpha,\beta}}\Gamma_T \right\} \land \left\{\frac{1}{|\mI_1|}X_{\mI_1}X_{\mI_1}^\top \succcurlyeq \frac{1}{C_{\alpha,\beta}}\Gamma_T \right\},
\end{align*}
and $\mE_{g,3}=\left\{\forall t:\ \|w_t\|\le \tO(d)\right\}$. Let $\mE_g \coloneqq \mE_{g,1}\land \mE_{g_2} \land \mE_{g,3}$. Then, $\Pr[\mE_g]\ge 1-\delta-o(1)$.
\end{lemma}
We will prove Lemma \ref{lem:Eg} in Appendix \ref{appdx:prf:eg}. We will need some definitions to ease notation. Let $B\coloneqq\mB(\alpha), X_0\coloneqq X_{\mI_0}$, $N_0\coloneqq |\mI_0|$, $\nu_t^*\coloneqq \Ex_{y\sim \mN(\sA x_t,I,\mS_{t+1})}[y]$, and $\zeta_t\coloneqq \nu_t^*-\sA x_t$. Finally, for all $\mI\subseteq [T]$, let $W_{\mI}, Z_{\mI}$ denote the $d\times |\mI|$ matrices whose columns are $(w_t)_{t\in \mI}$ and $(\zeta_t)_{t\in \mI}$ respectively (in temporal order). 
\par On $\mE_g$, $X_0X_0^\top \succ 0$, so the least-squares solution is unique, i.e., $ A_0=\sum_{t\in \mI_0}x_{t+1}x_t^\top \left(X_0X_0^\top\right)^{-1}$. Using $x_{t+1}=\sA x_t+w_t$, we get
\begin{align*}
    (A_0-\sA)\cdot \left(\frac{X_0X_0^\top}{N_0}\right)^{1/2}&=\frac{1}{\sqrt{N_0}}\sum_{t\in \mI_0}w_tx_t^\top  \left(X_0X_0^\top\right)^{-1/2} \\
    &= \underbrace{\frac{1}{\sqrt{N_0}}\sum_{t\in \mI_0'}(w_t-\zeta_t ) x_t^\top  \left(X_0X_0^\top\right)^{-1/2}}_{\Delta_1}+\underbrace{\frac{1}{\sqrt{N_0}}\sum_{t\in \mI_0'}\zeta_t x_t^\top\left(X_0X_0^\top\right)^{-1/2}}_{\Delta_2} \\  
    &\ \ \ \ \ \ \ \ \ \ \  \ \ \ \  \ \ \ \ \ \ \ \ \ \ \ \ \ \ \ \  \ \ \   \ \ \ \ \ \ \ \ \ \  \ \     \ \ \ \ \ \ \ \ +\underbrace{\frac{1}{\sqrt{N_0}}\sum_{t\in \mI_0\cap B}w_tx_t^\top\left(X_0X_0^\top\right)^{-1/2}}_{\Delta_3},
\end{align*}
where remember that $\mI_0'=\mI_0\setminus B$. We will control each $\|\Delta_i\|_{\text{F}}$ separately.
\begin{claim}\label{cl:delta3}
On $\mE_g$, $\|\Delta_3\|_{\text{F}}^2\le \tO(dL/N_0)$.
\end{claim}
\begin{claim}\label{cl:delta2}
On $\mE_g$, $\|\Delta_2\|_{\text{F}}^2\le O(\log(1/\alpha)+1)$.
\end{claim}
\begin{claim}\label{cl:delta1}
On $\mE_g$, with probability at least $1-\delta$, we have $\|\Delta_1\|_\text{F}^2\le \tO(d^2/N_0)$.
\end{claim}
Given these claims, we have that on $\mE_g$, with probability at least $1-\delta$, 
\begin{align*}
    \left\|(A_0-\sA)\cdot \left(\frac{X_0X_0^T}{N_0}\right)^{1/2} \right\|_{\text{F}}^2 \leq \tO(1)\cdot \frac{d^2+dL}{N_0} + O(\log(1/\alpha)+1)
\end{align*}
But, on $\mE_g$, $N_0\ge \alpha \beta T/4$ and $\frac{X_0X_0^\top}{N_0}\succcurlyeq \frac{1}{\poly\left(\frac{1}{\alpha \beta}\right)}\Gamma_T$. Setting $\delta=1/T$, and using the assumed lower bound for $T$, we have that with probability $1-o(1)$, $\|A_0-\sA\|_{\Gamma_T}\le O(\sqrt{\log(1/\alpha)}+1)$. Moreover, in Appendix \ref{cl:normalized}, we show that $\|\Gamma_T^{-1/2} x_t\| $ is $O(d)$-subgaussian. Combining with Lemma \ref{lem:Eg}, we complete the proof of Theorem \ref{thm:init:new}. We now prove the three claims.
\subsection{Proof of Claim \ref{cl:delta3}}

Let $B'= \mI_0\cap B$. We will show that on $\mE_g$, $\|\Delta_3\|_{\text{F}}^2\le \frac{1}{N_0}\sum_{t\in B'}\|w_t\|^2$. Observe that from Lemma \ref{lem:Eg}, this immediately gives the claim.
\begin{align*}
    \|\Delta_3\|_{\text{F}}^2& =\frac{1}{N_0}\trace \left(W_{B'}X_{B'}^\top (X_0X_0^\top)^{-1}X_{B'}W_{B'}^\top\right)\\
    &=\frac{1}{N_0} \Big\langle W_{B'}^\top W_{B'},X_{B'}^\top (X_0X_0^\top )^{-1}X_{B'} \Big\rangle \\
    & = \frac{1}{N_0}\Big\langle W_{B'}^\top W_{B'},X_{B'}^\top P_{X_{B'}}(X_0X_0^\top)^{-1}P_{X_{B'}}X_{B'}\Big\rangle
\end{align*}
where $P_{X_{B'}}$ is the projection matrix for the columnspace of $X_{B'}$. Now, note that $X_0X_0^T\succcurlyeq X_{B'}X_{B'}^T$. This implies that 
\begin{align*}
    P_{X_{B'}}\left(X_0X_0^\top \right)^{-1}P_{X_{B'}} \preccurlyeq (X_{B'}X_{B'}^\top)^+,
\end{align*}
where $(X_{B'}X_{B'}^T)^+$ denotes the pseudoinverse of $X_{B'}X_{B'}^T$. For a proof of this implication, see \cite{arora2019fine} Lemma E.1. Thus,
\begin{align*}
    \|\Delta_3\|_{\text{F}}^2& \le \frac{1}{N_0}\Big\langle W_{B'}^\top W_{B'},X_{B'}^\top (X_{B'}X_{B'}^\top)^+X_{B'} \Big\rangle.
\end{align*}
But, $X_{B'}^T (X_{B'}X_{B'}^\top )^+X_{B'}$ is just the projection matrix for the rowspace of $X_{B'}$, and so $ \|\Delta_3\|_{\text{F}}^2\le \frac{1}{N_0} \|W_{B'}\|_{\text{F}}^2$, and we are done. $\blacksquare$
\subsection{Proof of Claim \ref{cl:delta2}}

First we show that on $\mE_g$, if $t\in \mI_0'$, then $\|\zeta_t\|\le \sqrt{2\log(1/\alpha)}+1 $. Consider the filtration $\mG_t=\sigma(w_0,\dots,w_{t-1}, \mS_1,\dots,\mS_{t+1})$, and (similarly to the time-series case) the events $O_t=\{x_t\in \mS_t\}$, and 
\begin{align*}
    G_t\coloneqq \left\{O_t \rightarrow \left\{\mN(\sA x_t, I;\mS_{t+1})\ge \alpha \right\}\right\}
\end{align*}
Notice that $t\in\mI_0'$ implies $O_t\land G_t$. Now, given $\mG_{t}$ and $O_t\land G_t$, we have $x_{t+1}\sim \mN(\sA x_t,I ,\mS_{t+1})$, and $\mN(\sA x_t,I ;\mS_{t+1})\ge \alpha$. Using Claim \ref{cl:means:new}, we get $\|\zeta_t\|=\|\nu_t^*-\sA x_t\|\le \sqrt{2\log(1/\alpha)}+1$. Now,
\begin{align*}
\| \Delta_2 \|_{\text{F}}^2=\frac{1}{N_0} \Big\langle Z_{\mI_0'}^\top Z_{\mI_0'},X_{\mI_0'}^\top (X_0X_0^\top)^{-1}X_{\mI_0'} \Big\rangle,
\end{align*}
and since $X_{\mI_0'} X_{\mI_0'}^\top \preccurlyeq (X_0X_0^\top)$, the exact same arguments as in proof of Claim \ref{cl:delta3} show that $\|\Delta_3\|_{\text{F}}^2\le \frac{1}{N_0} \| Z_{\mI_0'}\|_{\text{F}}^2\le O(\log(1/\alpha)+1)$. $\blacksquare$

\subsection{Proof of Claim \ref{cl:delta1}}
First, remember that on $\mE_g,$ $X_0X_0^\top \succcurlyeq X_{\mI_0'}X_{\mI_0'}^\top\succ 0$. So,
\begin{align}\label{eq:forb-to-spec}
    \|\Delta_1\|_\text{F}&\le \sqrt{d/N_0}\cdot \left\|\left(W_{\mI_0'}-Z_{\mI_0'}\right)X_{\mI_0'}^\top \left(X_{0}X_{0}^\top\right)^{-1/2}\right\|_2 \\
    &\le \sqrt{d/N_0}\cdot \left\|\left(W_{\mI_0'}-Z_{\mI_0'}\right)X_{\mI_0'}^\top \left(X_{\mI_0'}X_{\mI_0'}^\top \right)^{-1/2}\right\|_2 
\end{align}
Now,
\begin{align*}
  \left\|\left(W_{\mI_0'}-Z_{\mI_0'}\right)X_{\mI_0'}^\top \left(X_{\mI_0'}X_{\mI_0'}^\top\right)^{-1/2}\right\|_2&= \sup_{v_1,v_2\in \mS^{d-1}}v_1^T \left(W_{\mI_0'}-Z_{\mI_0'}\right)X_{\mI_0'}^\top\left(X_{\mI_0'}X_{\mI_0'}^\top\right)^{-1/2} v_2 \\
  &= \sup_{v_1,v_2\in \mS^{d-1}}\frac{v_1^T \left(W_{\mI_0'}-Z_{\mI_0'}\right)X_{\mI_0'}^\top v_2}{\left\|\left(X_{\mI_0'}X_{\mI_0'}^\top\right)^{1/2} v_2\right\|_2}  \\
  &=\sup_{v_1,v_2\in \mS^{d-1}}\frac{v_1^T \left(W_{\mI_0'}-Z_{\mI_0'}\right)X_{\mI_0'}^\top v_2}{\left\|X_{\mI_0'}^\top v_2\right\|_2}\\
  &=\sup_{v_1\in \mS^{d-1},v_2\in \R_*^{d}}\frac{v_1^T \left(W_{\mI_0'}-Z_{\mI_0'}\right)X_{\mI_0'}^\top v_2}{\left\|X_{\mI_0'}^\top  v_2\right\|_2}.
\end{align*}
Fix some $v_1\in \mS^{d-1},v_2\in \R_*^{d}$. Let $H_t\coloneqq v_1^T(w_t-\zeta_t)\cdot \cha\{t\in I_0'\}$, and $U_t\coloneqq v_2^Tx_t \cdot \cha\{t\in I_0'\}$, and observe that on $\mE_g$,
\begin{align*}
   \frac{v_1^T \left(W_{\mI_0'}-Z_{\mI_0'}\right) X_{\mI_0'}^\top v_2}{\left\|X_{\mI_0'}^\top v_2\right\|_2}= \frac{\sumt H_t\cdot U_t}{\sqrt{\sumt U_t^2}}
\end{align*}
Let $\Gamma_{\text{min}}\coloneqq \frac{ 1}{C_{\alpha,\beta}}\Gamma_{T}$, $\Gamma_{\text{max}} \coloneqq C_{\alpha,\beta}\frac{d}{\delta }\Gamma_T$, and $C_{\alpha,\beta}=\poly(1/\alpha,1/\beta)$ that is large enough, so that $\Gamma_{\text{min}}\preccurlyeq X_{\mI_0'}X_{\mI_0'}^T \preccurlyeq \Gamma_{\text{max}} $ on $\mE_g$. Fix a $\delta\in(0,1)$, and consider a large enough $K= \widetilde{\Theta}(d)$. We will bound the probability
\begin{align*}
   p_1\coloneqq \Pr \left[\left\{\frac{\sumt H_t\cdot U_t}{\sqrt{\sumt U_t^2}}\ge K\right\} \land \left\{v_2^\top \Gamma_{\text{min}} v_2\le \sumt U_t^2 \le v_2^\top\Gamma_{\text{max}} v_2\right\} \right] .
\end{align*}
Consider the filtration $\mF_t\coloneqq \sigma(w_0,w_1,\dots,w_{t-1},\mS_1,\dots,\mS_{t+1},\cha\{x_{t+1}\in \mS_{t+1}\})$ (similarly to the time-series case). Given $\mF_t$ and that $t\in \mI_0'$, we have $\sA x_t+w_t =x_{t+1} \sim \mN(\sA x_t,I,\mS_{t+1})$, and $\mN(\sA x_t,I;\mS_{t+1})\ge \alpha$. Combining with Claim \ref{appdx:prob1} (Appendix \ref{appdx:prob}), we get $\|v_1^Tw_t\|_{\psi_2}^2\le O(1\lor \log{(1/\alpha)})$. Also, from Claim \ref{cl:means:new}, $\|v_1^T\zeta_t\|_{\psi_2}^2\le \|\zeta_t\|_{2}^2\le O(1\lor \log{(1/\alpha)})$. Thus, given $\mF_t$, $H_t$ is mean zero, and $O(1\lor \log{(1/\alpha)})$-subgaussian. Furthermore, observe that $H_t$ is $\mF_{t+1}$-measurable, while $U_t$ is $\mF_{t}$-measurable. Using Lemma 4.2 from \cite{simchowitz2018learning}, we get 
\begin{align*}
   p_1\le \log \left(\frac{v_2^\top \Gamma_{\text{min}} v_2}{v_2^\top \Gamma_{\text{max}} v_2}\right) \cdot \exp \left(-\frac{K^2}{O(1\lor \log{(1/\alpha)})}\right).
\end{align*}
Also, observe that 
\begin{align*}
   \Pr\left[ \frac{v_1^\top \left(W_{\mI_0'}-Z_{\mI_0'}\right)X_{\mI_0'}^\top v_2}{\left\|X_{\mI_0'}^\top v_2\right\|_2} \cdot \cha\{\mE_g\}> K \right] \le p_1
\end{align*}
Now, let $\mN_1$ be a $1/2$-net of $\mS^{d-1}$, and let $\mN_2$ be a $1/4$-net of $\mS_{\Gamma_{\text{min}}}$ in the norm $\|\Gamma_{\text{max}}^{1/2}(\cdot)\|_2$. Taking a union bound,
\begin{align*}
   \Pr\left[\exists\ v_1\in \mN_1,v_2\in \mN_2:\  \frac{v_1^\top \left(W_{\mI_0'}-Z_{\mI_0'}\right)X_{\mI_0'}^\top v_2}{\left\|X_{\mI_0'}^\top v_2\right\|_2} \cdot \cha\{\mE_g\}> K \right] \le p_1\cdot |\mN_1|\cdot |\mN_2|
\end{align*}
Using the bounds on $|\mN_1|, |\mN_2|$ from \cite{simchowitz2018learning} (Lemma D.1), we get
\begin{align*}
 p_1\cdot |\mN_1|\cdot |\mN_2|&\le  \log\left(\frac{v_2^\top\Gamma_{\text{max}} v_2}{v_2^\top\Gamma_{\text{min}} v_2}\right) \cdot \exp \left(O(d)+2\log \det(\Gamma_{\text{max}}\Gamma_{\text{min}}^{-1})-\frac{K^2}{O(1\lor \log{(1/\alpha)})}\right) \le \delta,
\end{align*}
where we used the definitions of $\Gamma_{\text{max}},\Gamma_{\text{min}}$, and $K$. So now we have
\begin{align*}
 \Pr\left[\left\{\sup_{ v_1\in \mN_1,v_2\in \mN_2}  \frac{v_1^\top \left(W_{\mI_0'}-Z_{\mI_0'}\right)X_{\mI_0'}^\top v_2}{\left\|X_{\mI_0'}^\top v_2\right\|_2} > K \right\} \land \mE_g \right] \le \delta.  
\end{align*}
With an argument identical to the one of \cite{simchowitz2018learning} (Appendix D.2), we get that 
\begin{align*}
 \Pr\left[\left\{\sup_{ v_1\in \mS^{d-1},v_2\in \R_*^d}  \frac{v_1^\top \left(W_{\mI_0'}-Z_{\mI_0'}\right) X_{\mI_0'}^\top v_2}{\left\|X_{\mI_0'}^\top v_2\right\|_2} > 4K \right\}\land \mE_g\right] \le \delta,
\end{align*}
and we are done. $\blacksquare$

\subsection{ Proof of Lemma \ref{lem:Eg}}\label{appdx:prf:eg}
We will deduce Lemma \ref{lem:Eg} as a corollary of the following lemma.
\begin{lemma}\label{lem:cov}
Fix $\delta,\epsilon\in(0,1)$, and an integer $k\ge 1$. Then, there exist absolute constants $C,C'>0$, such that if
\begin{align}\label{eq:k-constr}
    \frac{T}{k}\ge C\left( \frac{d}{\epsilon^2}\log{\left(\frac{d}{\delta\epsilon}\right)}+\frac{1}{\epsilon^2}\log\det (\Gamma_T \Gamma_{k}^{-1})\right), \ \ \text{then}
\end{align}
\begin{align}\label{eq:covariance-general}
    \Pr\left[\left\{\frac{1}{T}XX^\top \preccurlyeq  \frac{d}{\delta}\Gamma_T\right\}\land \left\{ \forall\   \mI\subseteq [T]\ \text{s.t.}\ |\mI|\ge \epsilon T,\ \text{we have}\   \frac{1}{|\mI|}X_{\mI}X_{\mI}^\top\succcurlyeq \frac{ \epsilon^4}{C'}\Gamma_{k} \right\}\right] \ge 1-\delta
\end{align}
\end{lemma}
Given the above, we prove Lemma \ref{lem:Eg}.
\begin{proof}
We will use Lemma \ref{lem:cov}, with $\epsilon=\alpha \beta/8$, and $k=\widetilde{\Theta}\left(\frac{1}{1-\rho(\sA)}\right)$. We need the following proposition.
\begin{proposition}
For large enough $k=\widetilde{\Theta}\left(\frac{1}{1-\rho(\sA)}\right)$, we have 
$2\Gamma_k\succcurlyeq \Gamma_T$.
\end{proposition}
\begin{proof}
$\|\Gamma_T-\Gamma_k\|_2\le \sum_{s=k}^{T-1}\left\|(\sA^{s})(\sA^s)^{\top}\right\|_2$. Also, observe that $\left\|(\sA^{s})(\sA^s)^{\top}\right\|_2\le \text{cond}(S)^2 \rho^{2k}(\sA)$. Summing up and using the scale of $k$, we get $\|\Gamma_T-\Gamma_k\|_2<1$. The fact that $\Gamma_k \succcurlyeq I$ completes the proof.
\end{proof}
Observe that the above proposition, and the assumed lower bound on $T$ imply that $\ref{eq:k-constr}$ is satisfied. Now, from Proposition (\ref{prop:largeN}), $\Pr[|\mI_0|,|\mI_1|\ge \alpha \beta T/4]\ge 1-o(1)$. Also, from Assumption \ref{ass:survival-prob}, $\Pr[|\mB(\alpha)|\le L]\ge 1-o(1)$. Thus, $\Pr[|\mI_0|,|\mI_1|,|\mI_0'|\ge \epsilon T]\ge 1-o(1)$ (remember that $L<<T$). Lemma \ref{lem:cov} and Gaussian-norm concentration finish the proof.
\end{proof}
The main part of this section is devoted to proving the following lemma, from which we can easily deduce Lemma \ref{lem:cov}.
\begin{lemma}\label{lem:fix-v}
Fix $\epsilon\in(0,1)$, $1\le k\le \sqrt{\epsilon}T$, and $v\in\mS^{d-1}$. Let $c(\epsilon)=\frac{\sqrt{8/\pi}}{\epsilon}$, and $k(\epsilon)=\floor{\frac{\epsilon k}{2}}$. Then,
\begin{align*}
    \Pr\left[\exists\  \mI\subseteq [T]:\ |\mI|\ge 5\sqrt{\epsilon }T\ \ \text{and}\ \ \frac{1}{|\mI|}\sum_{t \in \mI}(v^{\top}x_t)^2 < \frac{v^\top \Gamma_{k(\epsilon)}v}{2c^2(\epsilon)}\right] \le \exp\left(- \frac{\floor{T/k}}{8}\right).
\end{align*}
\end{lemma}
Given Lemma \ref{lem:fix-v}, we can show the following corollary, from which Lemma \ref{lem:cov} immediately follows.
\begin{corollary}
Fix $\epsilon\in (0,1)$, $1\le k\le \sqrt{\epsilon}T$, and let $c(\epsilon),k(\epsilon)$ as in Lemma \ref{lem:fix-v}. Then,
\begin{align*}
 \Pr\Bigg[\left\{\frac{1}{T}XX^\top \preccurlyeq  \frac{d}{\delta}\Gamma_T\right\}\land \bigg\{ \forall \  \mI\subseteq [T]\ \text{s.t.}\ |\mI|&\ge 5\sqrt{\epsilon }T\ \text{we have}\  \frac{1}{|\mI|}X_{\mI}X_{\mI}^\top\succcurlyeq  \Gamma_{\text{min}}\bigg\}\Bigg]\\
&    \ge 1-\delta -\exp\left(d\log{9}+\log\det (\Gamma_{\text{max}}\Gamma_{\text{min}}^{-1})- \frac{\floor{T/k}}{8}\right), 
\end{align*}
where $\Gamma_{\text{min}}=\frac{ 1}{4c^2(\epsilon)}\Gamma_{k(\epsilon)}$, $\Gamma_{\text{max}}=\frac{d}{5\delta\sqrt{\epsilon}}\Gamma_T$.
\end{corollary}
\begin{proof}
Using that for any $\mI\subseteq [T]$ such that $ |\mI|\ge 5\sqrt{\epsilon }T$, we have $\frac{1}{|\mI|}X_{\mI}X_{\mI}^\top\preccurlyeq \frac{1}{5\sqrt{\epsilon}}\frac{1}{T}XX^\top$, the corollary follows from Lemma \ref{lem:fix-v} by using exactly the same (net-based) arguments as \cite{simchowitz2018learning} (Lemma D.1., and beginning of Section 4).
\end{proof}
For the rest of the section, we work on proving Lemma \ref{lem:fix-v}. We use the martingale small-ball technique of \cite{simchowitz2018learning}.
\begin{definition}[Martingale Small-Ball; \cite{simchowitz2018learning}]\label{def:bmsb} Let $\mG_t$ be a filtration, and let $(z_t)_{t \ge 1}$ be a $\{\mG_t\}_{t \ge 1}$-adapted random process taking values in $\R$. We say $(z_t)_{t\ge 1}$ satisfies the $(k,\nu,p)$-block martingale small-ball (BMSB) condition if, for any $j \ge 0$, one has $\frac{1}{k}\sum_{i=1}^k \Pr( |z_{j+i}\ |\ \ge \nu | \mG_{j}) \ge p$ almost surely.  Given a process $(x_t)_{t \ge 1}$ taking values in $\R^d$, we say that it satisfies the $(k,\Gamma_{\text{sb}},p)$-BMSB condition for $\Gamma_{\text{sb}} \succ 0$ if, for any fixed $v\in \mS^{d-1}$, the process $z_t\coloneqq \langle v, x_t\rangle$ satisfies $(k,\sqrt{v^\top \Gamma_{\text{sb}} v},p)$-BMSB.
	\end{definition}
	
	Now, suppose that a (real) process $(z_t)_t$ satisfies the $(k,\nu,1-\epsilon)$-BMSM condition. We partition the sequence $z_1,\dots,z_T$ into $\tau\coloneqq\floor{T/k}$ blocks of size $k$ (we discard the remaining terms). For $0\le\ell\le \tau-1$, we consider the events
	\begin{align*}
	    \mE_\ell\coloneqq \left\{\frac{1}{k}\sum_{j=1}^k\cha\left\{|Z_{\ell k+j}|\ge \nu \right\}\ge 1-\sqrt{\epsilon}\right\}
	\end{align*}
	\begin{claim}\label{cl:bmsb}
	$\Pr\left[\sum_{\ell=0}^{\tau-1}\cha\{\mE_\ell\} \le (1-\sqrt{\epsilon}/2)\tau\right]\le \exp(-\tau/8)$.
	\end{claim} 
	We provide a proof of Claim \ref{cl:bmsb} in Appendix \ref{appdx:prob} (Claim \ref{cl:bmsb-prob}).
	\subsubsection*{Application to LDSs}
	\begin{proposition}\label{prop:linear_system_small_ball_full} Consider the linear dynamical system $x_{t+1} = \sA x_t + w_t $, where $x_0= 0$, $w_t  \simiid \mN(0, I)$, and let  $\Gamma_t :=  \sum_{s=0}^{t-1} (\sA^{s})(\sA^s)^{\top}$. Fix an $\epsilon\in(0,1)$. Then, for $1 \le k \le T$, the process $( x_t )_{t \ge 1}$ satisfies the
\begin{align*}
   \left(k,\frac{1}{c^2(\epsilon)}\Gamma_{k(\epsilon)},1-\epsilon \right)\text{-block martingale small-ball condition},
\end{align*}
where $c(\epsilon)=\frac{\sqrt{8/\pi}}{\epsilon}$ and $k(\epsilon)=\floor{\frac{\epsilon k}{2}}$.
\end{proposition}
\begin{proof}
Consider the filtration $\mG_t\coloneqq \sigma(w_0,w_1,\dots,w_{t-1})$. Fix a $v\in \mS^{d-1}$. We have 
\begin{align*}
   v^\top x_{s+j}=v^\top \sA^jx_s +v^\top \sum_{i=0}^{j-1}\sA^iw_{s+j-i-1}.
\end{align*}
Given $\mG_s$, $v^\top \sum_{i=0}^{j-1}\sA^iw_{s+j-i-1}\sim \mN(0,v^
\top \Gamma_j v)$. From Claim \ref{cl:gaus-antic} (Appendix \ref{appdx:prob}), for any $c>0$,
\begin{align*}
    \Pr\left[|v^\top x_{s+j}|< \frac{v^
\top \Gamma_j v}{c}\ \bigg|\ \mG_s \right]\le \frac{\sqrt{2/\pi}}{c}.
\end{align*}
We choose $c=\frac{2\sqrt{2/\pi}}{\epsilon}$. Using that for all $j\ge k(\epsilon)$, $\Gamma_j \succcurlyeq \Gamma_{k(\epsilon)}$, we get
\begin{align*}
   \frac{1}{k}\sum_{j=0}^{k} \Pr\left[|v^\top x_{s+j}|\ge \frac{\sqrt{v^
\top \Gamma_{k(\epsilon)} v}}{c}\ \Bigg|\ \mG_s \right]& \ge \frac{1}{k}\sum_{j=k(\epsilon)}^{k} \Pr\left[|v^\top x_{s+j}|\ge \frac{\sqrt{v^
\top \Gamma_{k(\epsilon)} v}}{c}\ \Bigg|\ \mG_s \right] \\
&\ge \frac{1}{k}\sum_{j=k(\epsilon)}^{k} \Pr\left[|v^\top x_{s+j}| \ge \frac{\sqrt{v^
\top \Gamma_j v}}{c}\ \Bigg|\ \mG_s \right] \\
&\ge \frac{k-k(\epsilon)}{k}\left(1-\frac{\sqrt{2/\pi}}{c}\right)\ge (1-\epsilon/2)^2\ge 1-\epsilon.
\end{align*}
\end{proof}

Now, fix a $v\in \mS^{d-1}$, a $k\in [T]$, and an $\epsilon \in (0,1)$. Let $\sigma^2(v,\epsilon)\coloneqq \frac{1}{c^2(\epsilon)}v^\top \Gamma_{k(\epsilon)}v$. We define the events: 
\begin{align*}
\mE_\ell\coloneqq \left\{\frac{1}{k} \sum_{j=1}^{k}\mathds{1}\left\{ (v^\top x_{\ell k+j})^2\ge \sigma^2(v,\epsilon)p\right\} \ge 1 - \sqrt{\epsilon} \right\}, 
\end{align*}
for all $0 \le \ell \le \tau - 1$. Observe that Claim \ref{cl:bmsb} and Proposition \ref{prop:linear_system_small_ball_full} imply
\begin{align*}
    \Pr\left[\sum_{\ell =0}^{\tau -1}\mathds{1}\left\{\mE_\ell\right\}\le (1-\sqrt{\epsilon}/2) \tau \right] \le \exp\left(- \frac{\tau}{8}\right),
\end{align*}
which implies 
\begin{align*}
    \Pr\left[\sum_{t =1}^{T}\mathds{1}\left\{(v^{\top}x_t)^2\ge \sigma^2(v,\epsilon)\right\}\le (1-\sqrt{\epsilon}/2)(1-\sqrt{\epsilon})k \tau \right] \le \exp\left(- \frac{\tau}{8}\right).
\end{align*}
For $k\le\floor{\sqrt{\epsilon}T}$, we have $T-\tau k\le \sqrt{\epsilon}T$, and so 
\begin{align*}
    \Pr\left[\sum_{t =1}^{T}\mathds{1}\left\{(v^{\top}x_t)^2\ge \sigma^2(v,\epsilon)\right\}\le (1-5\sqrt{\epsilon}/2)T \right] \le \exp\left(- \frac{\tau}{8}\right).
\end{align*}
Observe that this implies
\begin{align*}
    \Pr\left[\exists\  \mI\subseteq [T]:\ |\mI|\ge 5\sqrt{\epsilon }T\ \ \text{and}\ \ \frac{1}{|\mI|}\sum_{t \in \mI}(v^{\top}x_t)^2 < \frac{\sigma^2(v,\epsilon)}{2}\right] \le \exp\left(- \frac{\tau}{8}\right).\ \ \blacksquare
\end{align*}

\section{Technical Claims}\label{appdx:prob}
\begin{claim}\label{appdx:prob1}
Let $z\sim \mN(\mu,I,S)$, and $\gamma=\mN(\mu,I;S)$. Then, $\|z-\mu\|_{\psi_2}^2\le O(1\lor \log{(1/\gamma)})$.
\end{claim}
\begin{proof}
It suffices to show that for some absolute constant $c > 0$, and for all $v \in \mS^{d-1}$,
\begin{align}\label{myrto3}
    \Pr\left[|v^\top z - v^\top\mu |\ge t\right] \le \exp\left(-\frac{t^2}{c \cdot \left(1\lor \log\left(\frac{1}{\gamma}\right)\right)}\right)
\end{align}
Observe that it suffices to show it for $v=(1,0,\cdots,0)$ (otherwise we can just change coordinates). So, it suffices to prove
\begin{align*}
    \Pr\left[|z_1 -\mu_1 |\ge t\right] \le \exp\left(-\frac{t^2}{c \cdot \left(1\lor \log\left(\frac{1}{\gamma}\right)\right)}\right)
\end{align*}
Now, consider samples $\xi_1,\dots,\xi_{\ceil{\frac{1}{\gamma}}} \simiid \mN(\mu,I) $. From \cite{van2014probability},
\begin{align*}
\text{$\forall x>0$,}\ \Pr\left[\max_i|\xi_{i,1} - \mu_1|\ge \sqrt{2\log\frac{1}{\gamma}} + x\right]\le 2 e^{-x^2/2}.
\end{align*}
Also, note that $\Pr\left[\exists i: \xi_i \in S\right] \ge 1 - (1 - \gamma)^{\ceil{\frac{1}{\gamma}}}>1/2$.
Let 
\begin{align*}
    p \coloneqq \Pr\left[|z_1 - \mu_1| \ge  t\right]= \Pr\left[|z_1 - \mu_1| \ge  \sqrt{2\log\frac{1}{\gamma}}+t-\sqrt{2\log\frac{1}{\gamma}}\right].
\end{align*}
For $t \ge \sqrt{2\log\frac{1}{\gamma}}$, let $x = t - \sqrt{2\log\frac{1}{\gamma}} \ge 0$, and so $p = \Pr  \left[|z_1 - \mu_1| \ge \sqrt{2\log\frac{1}{\gamma}} + x \right]$, and
\begin{align*}
    \Pr\left[\max_i |\xi_{i,1} - \mu_1| \ge \sqrt{2\log\frac{1}{\gamma}} + x \right]&= 1 - \prod_i \Pr \left[|\xi_{i,1} - \mu_1| < \sqrt{2\log\frac{1}{\gamma}} + x \right] \\
    &= 1 - \prod_i\left(1-\Pr\left[|\xi_{i,1} - \mu_1| \ge \sqrt{2\ln\frac{1}{\gamma}} + x \right]\right)\\
    &\ge 1 - \prod_i\left(1-\Pr\left[\xi_i \in S\right]\cdot\Pr\left[|\xi_{i,1} - \mu_1| \ge \sqrt{2\ln\frac{1}{\gamma}} + x\  \Big|\  \xi_i \in S\right]\right)
\end{align*}
Now, $\Pr\left[\xi_i \in S\right] = \gamma$, and given $\xi_i \in S$, we have $\xi_i \sim \mN(\mu,I,S)$, and so observe that 
\begin{align*}
    \Pr\left[|\xi_{i,1} - \mu_1|\ge \sqrt{2\ln\frac{1}{\gamma}} + x \ \bigg|\ \xi_i\in S\right] = p, \ \text{Thus,}
\end{align*}
Thus, 
\begin{align*}
    2e^{-x^2/2}\ge\Pr\left[\max_i |\xi_{i,1} - \mu_1| \ge \sqrt{2\ln\frac{1}{\gamma}} + x \right]
    \ge 1 - \prod_i\left( 1 - \gamma p \right) 
    = 1 - \left( 1 - \gamma p \right)^{\ceil{\frac{1}{\gamma}}}\ge  1 -e^{-p}
\end{align*}
Now, observe that due to $e^{- \theta}$ being convex in $[0,1]$, we have $e^{-\theta}\le 1-\theta(1-1/e)$, which allows us to conclude that $p\le 4e^{-x^2/2}$. Let $c(\gamma)=\sqrt{2\log(1/\gamma)}$. We showed that for $t\ge c(\gamma)$,
\begin{align}\label{myrto5}
    \Pr\left[ |z_1 - \mu_1| \ge t \right] \le 2\cdot\exp\left(-\frac{\left(t-c(\gamma)\right)}{2}^2+\log 2\right).
\end{align}
From here, straightforward calculations complete the proof.
\end{proof}

\begin{claim}\label{appdx:prob2}
Let $X$ be a random vector in $\R^d$, such that $\|X\|_{\psi_2}^2\le \sigma^2$. Then, $\|X\|_2$ is $\O(\sigma^2 d\log{(2d)})$-subgaussian.
\end{claim}
\begin{proof}
For all $i$, $\Pr\left[x_i^2\ge t^2\right]\le2\cdot\exp\left(-\Omega(t^2/\sigma^2)\right)$. Letting $t=\Theta(\sigma)\cdot\sqrt{\log\left(\frac{2d}{\delta}\right)}$ and applying union-bound finish the proof.
\end{proof}
\begin{claim}\label{appdx:prob3-max}
Let $M$ be the maximum of $T$ nonnegative $\sigma^2$-subgaussian random variables (not necessarily independent). Then, for all $q\ge 1$,
\begin{align*}
    \Ex[M^{q}]\le \left(\sqrt{2\sigma^2\log{T}}+C\sigma \sqrt{q}\right)^q,
\end{align*}
where $C>0$ is an absolute constant.
\end{claim}
\begin{proof}
Let $B:=\sqrt{2\sigma^2\log T}$. From \cite{van2014probability}, $\Pr[M\ge B+x]\le e^{-x^2/(2\sigma^2)}$. Let $Y\coloneqq M\lor B - B\ge 0$. For all $x>0$, $\Pr[Y\ge x]\le e^{-x^2/(2\sigma^2)}$. This implies that for all $p\ge 1$, $\|Y\|_p \le O(\sigma \sqrt{p})$ (\cite{vershynin2018high}). Thus, $\|M\|_p \le \|M\lor B\|_p \le B + O(\sigma)\sqrt{p}$.
\end{proof}
\begin{claim}\label{cl:chernoff}
$\Pr[\neg C_t\ ]\le 1/T^2$.
\end{claim}
\begin{proof}
Given $\mF_t$ and $O_t$, let $i=i(t)$, and consider $p$, $k$, and the $\xi_j$'s of the Test at iteration $i$. Then $\mathbb{E}[p\ |\ \mF_t, O_t]=\gamma_t$. We consider two cases, and we use standard Chernoff-bounds in both of them.
\begin{enumerate}
    \item If $\gamma_t> 4\gamma$, then $\Pr[p< 2\gamma]\le \exp(-k\gamma/2)\le 1/T^2$.
    \item If $\gamma_t < \gamma$, then $\Pr[p\ge 2\gamma]\le  \exp \left(-\frac{(2\gamma/\gamma_t)^2}{2+2\beta/\gamma_t}\cdot \gamma_t \cdot k\right)\le \exp(-k \gamma )\le 1/T^2$.
\end{enumerate}
\end{proof}
\begin{claim}\label{cl:gaus-antic}
Let $z\sim \mN(\mu,\sigma^2)$ be a one-dimensional normal. Fix a $c>0$. Then, $\Pr[|z|<\sigma/c]\le \frac{\sqrt{2/\pi}}{c}$.
\end{claim}
\begin{proof}
$\Pr[|z|\le \sigma/c]=\frac{1}{\sqrt{2\pi}\sigma}\int_{-\sigma/c}^{\sigma/c}e^{-(x-\mu^2)/(2\sigma^2)}dx\le \frac{1}{\sqrt{2\pi}\sigma}\int_{-\sigma/c}^{\sigma/c} dx\le  \frac{\sqrt{2/\pi}}{c}$.
\end{proof}

	\begin{claim}\label{cl:bmsb-prob}
	$\Pr\left[\sum_{\ell=0}^{\tau-1}\cha\{\mE_\ell\} \le (1-\sqrt{\epsilon}/2)\tau\right]\le \exp(-\tau/8)$.
	\end{claim} 
	\begin{proof}
	Let $\mH_t = \sigma(z_1,\dots,z_t)$. Fix a block $\ell$. From $(k,\nu,1 -\epsilon )$-BMSB, $\frac{1}{k} \sum_{j=1}^{k}\Pr[|Z_{\ell k+j}| \ge \nu \ |\ \mH_{\ell k}] \ge 1 - \epsilon$, i.e.,
\begin{align}\label{eq:sb-exp}
  \frac{1}{k}   \Ex\left[\sum_{j=1}^{k}\cha\left\{|Z_{\ell k+j}|\ge \nu  \right\}\  \bigg|\ \mH_{\ell k}\right] \ge 1 - \epsilon 
\end{align}
We will need the following simple proposition.
\begin{proposition}
$v$
Let $Z$ be a random variable supported on $[0,1]$, such that  $\Ex[Z] \ge 1 - \epsilon$, for some $\epsilon \in (0,1)$. Then, we have $\Pr[Z\ge 1 - \sqrt{\epsilon}]\ge 1 - \sqrt{\epsilon}$. 
\end{proposition}
\begin{proof}
Suppose it is not true, then
\begin{align*}
    \Ex[Z] = \int_{0}^{1} \Pr[Z \ge x] \,dx  =\int_{0}^{1-\sqrt{\epsilon}} \Pr[Z \ge x] \,dx + \int_{1-\sqrt{\epsilon}}^{1} \Pr[Z \ge x] \,dx &< 1-\sqrt{\epsilon} + \sqrt{\epsilon} \cdot (1-\sqrt{\varepsilon}) \\
    & = 1- \epsilon,
\end{align*}
contradiction.
\end{proof}
Combining (\ref{eq:sb-exp}) with the proposition above,
\begin{align}\label{myrto-ex}
    \Pr[\mE_\ell\ |\ \mF_{\ell k}] = \Pr\left[\frac{1}{k} \sum\limits_{j=1}^{k}\mathds{1}\left\{|Z_{\ell k+j}| \ge \nu \right\} \ge  1 - \sqrt{\epsilon}\ \Bigg|\ \H_{\ell k}\right] \ge 1 - \sqrt{\epsilon}.
\end{align}
Now, let  $\Delta_\ell = \mathds{1}(\mE_\ell) - \Pr[\mE_\ell\ |\ \mR_\ell],$ where $\mR_\ell \coloneqq \mH_{\ell k}$, so that $\Ex[\Del_\ell\ |\ \mR_\ell]$ = 0, and $\Delta_\ell$ is $\mH_\ell$-measurable. Since $-1 \le \Delta_\ell \le 1$, from Azuma-Hoeffding inequality,
\begin{align*}
    &\Pr\left[ \sum\limits_{\ell =0}^{\tau -1}\Del_\ell\le -t\right]\le \exp\left(- \frac{t^{2}}{2\tau}\right)
\end{align*}
Setting $t=\tau(1-\sqrt{\epsilon}/2)$, and combining with (\ref{myrto-ex}) completes the proof.
\end{proof}

\subsection{Proof of Claim \ref{cl:cov:new}}\label{appdx:prf:cov}
We will prove the following claim, which combined with Claim \ref{prob-general} (Appendix \ref{appdx:prf:prob-lb}) gives Claim \ref{cl:cov:new}.
\begin{claim}
If $\mN(\mu,I;S)\ge \zeta$, then $ \text{Cov}_{z\sim\mN(\mu,I,S)}[z,z] \succcurlyeq \frac{\zeta^3}{C}I$, where $C>0$ is an absolute constant.
\end{claim}
\begin{proof}
The proof is implicitly in \cite{daskalakis2018efficient}. 
Let $R=\text{Cov}_{z\sim\mN(\mu,I,S)}[z,z]$ and
\begin{align*}
  R'=\Ex_{z\sim\mN(0,I)}\left[\left(z-\Ex_{z\sim\mN(\mu,I,S)}[z]\right)\left(z-\Ex_{z\sim\mN(\mu,I,S)}[z]\right)^T\right]  
\end{align*}
First, from Claim 3 in \cite{daskalakis2018efficient}, $R'\succcurlyeq I$. Fix some $v\in S^{d-1}$. We want to relate $v^{T}R v$ with $v^{T}R' v$. Now, $v^{T}R v= \Ex_{z\sim\mN(\mu,I,S)}[p_v(z)]$, and $v^{T}R' v=\Ex_{z\sim\mN(\mu,I)}[p_v(z)]$, 
where $p_v(z)$ is a polynomial of degree at most two, whose coefficients depend on $v$. Also, $p_v(z)\ge 0$, and $R'\succcurlyeq I$ implies
\begin{align*}
\Ex_{z\sim\mN(\mu,I)}[p_v(z)] \ge 1
\end{align*}
 Let $\theta\coloneqq\left(\frac{\zeta}{4C_*}\right)^2$, where $C_*$ is the constant from Theorem 5 in \cite{daskalakis2018efficient}. Let $\overline{S}=\{z\in \R^d:\ p_v(z)\le \theta\}$. Applying the aforementioned theorem, 
\begin{align*}
  \mN(\mu,I,\overline{S})\le \frac{C_*\cdot2\cdot\frac{\zeta}{4C_*}}{\left(\Ex_{z\sim\mN(\mu,I)}[p_v(z)]\right)^{1/2}}   \le \zeta/2.
\end{align*}
Thus, $\Ex_{z\sim \mN(0,I,S)}[p_v(z)]\ge\zeta/2\cdot \theta \ge \Omega(\zeta^3)$.
\end{proof}

\subsection{Proof of Claim \ref{cl:prob-lb:new}}\label{appdx:prf:prob-lb}
We prove a more general claim:
\begin{claim}\label{prob-general}
If $ \mN(\mu^*,I;S)\ge \alpha$ and $\| \mu - \mu^*\| \le r$, then $\mN(\mu,I;S)\ge \frac{\alpha}{2}\cdot\exp\left(-\frac{r^2}{2}-r\cdot\sqrt{2\log\frac{1}{\alpha}}\right)$.
\end{claim}
Observe that if $r=c\cdot s(\alpha)$, we get Claim \ref{cl:prob-lb:new}.
\begin{proof}
Without loss of generality, suppose $\mu^*=0$ (the general case follows from a simple affine transformation). 
\begin{align}\label{eqa}
     \mN(\mu,I;S)=\Ex_{x\sim\mN(0,I)}\left[\mathds{1}(x\in S)\cdot\frac{\mN(\mu,I;x)}{\mN(0,I;x)}\right]
\end{align}
We will show that if $x\sim \mN(0,I)$ then with probability at least $1-\frac{\alpha}{2}$ the ratio $\frac{\mN(\mu,I;x)}{\mN(0,I;x)}$ is larger than some bound $\kappa$. Since $\mN(0,I;S)\ge \alpha$, this implies from (\ref{eqa}) that $\mN(\mu,I;S)\ge \frac{\alpha}{2}\cdot\kappa$.
\begin{align}\label{eqb}
\frac{\mN(\mu,I;x)}{\mN(0,I;x)}=\exp\left(-\frac{\|\mu\|^2}{2}+x^\top\mu\right)
\ge\exp\left(-r^2/2\right)\cdot\exp\left(x^\top \mu\right)
\end{align}
Observe now that since $x\sim\mN(0,I)$, we have $x^\top\mu\sim\mN(0,\|\mu\|^2),$ so for all $ t\ge 0$,
\begin{align*}
    \Pr\left[x^\top \mu\le-\|\mu\|t\right]\le \exp\left(-t^2/2\right)
\end{align*}
Letting $t=\sqrt{2}\sqrt{\log\frac{2}{2\alpha}}$, we get $\Pr\left[x^\top\mu\le-\|\mu\|t\right]\le\frac{\alpha}{2}$, and we are done.
\end{proof}
\subsection{Bound on $R_x$}\label{cl:Rx}
\begin{claim}
For all $t$, $\Ex[\|x_t\|^2]\le O\left(\frac{d}{1-\rho(\sA)}\cdot \text{cond}(U)^{2}\right)$
\end{claim}
\begin{proof}
By unrolling the system, $x_t=\sum_{k=0}^{t-1}\sA^kw_{t-k-1}$, and so $x_t\sim \mN\left(0,\sum_{k=0}^{t-1}\sA^k{\left(\sA^k\right)}^\top \right)$, so $\Ex[\|x_t\|^2]\le \trace(\Gamma_T)$.
 Since $\sA=UDU^{-1}$, we have $\sA^k{\left(\sA^k\right)}^\top=UD^kU^{-1}U^{-*}D^kU^*$, and so
\begin{align*}
    \trace{\left(\sA^k{\left(\sA^k\right)}^\top\right)}&=\trace{\left(UD^kU^{-1}U^{-*}D^{k*}U^*\right)}=\trace{\left(U^*UD^kU^{-1}U^{-*}D^{k*}\right)} \\
    &\le \|U^*U\|_2\cdot \trace{\left(D^kU^{-1}U^{-*}D^{k*}\right)}=\|U^*U\|_2\cdot\trace{\left(D^{k*}D^kU^{-1}U^{-*}\right)} \\
    &\le \|U^*U\|_2 \cdot \|U^{-*}U^{-1}\|_2\cdot \trace{\left(D^{k*}D^k\right)}= \text{cond}(U)^2\cdot 
    \trace{\left(D^{k*}D^k\right)}
    \\
    &\le \text{cond}(U)^2\cdot d\cdot \rho(\sA)^{2k}.
\end{align*}
Summing over $t$ completes the proof.
\end{proof}
\subsection{Bound on $\left\|\Gamma_T^{-1/2}x_{t}\right\|$.}\label{cl:normalized}
\begin{claim}
For all $t$, $\left\|\Gamma_T^{-1/2}x_{t}\right\|$ is $O(d)$-subgaussian.
\end{claim}
\begin{proof}
By unrolling the system, $x_t=\sum_{k=0}^{t-1}\sA^kw_{t-k-1}$, and so $x_t\sim \mN\left(0,\sum_{k=0}^{t-1}\sA^k{\left(\sA^k\right)}^\top\right)$. Since $\ga_T\succcurlyeq \sum_{k=0}^{t-1}\sA^k{\left(\sA^k\right)}^T$, $\|\ga_t^{-1/2}x_t\|$ is stochastically dominated by a $\|z\|$, where $z\sim \mN(0,I)$. So, from concentration of Gaussian norm (\cite{vershynin2018high}), $\left\|\Gamma_T^{-1/2}x_{t}\right\|$ is $O(d)$-subgaussian.
\end{proof}